\documentclass[11pt,draftcls,onecolumn]{IEEEtran}

\usepackage[utf8x]{inputenc}
\usepackage{graphicx,epsfig,epstopdf}
\usepackage{stmaryrd}
\usepackage{amssymb}
\usepackage{amsfonts}
\usepackage{latexsym}
\usepackage{amsmath}
\usepackage{mathrsfs}
\usepackage{bbm}
\usepackage{amsthm}
\usepackage{color}
\usepackage{soul}
\usepackage{cite}
\usepackage{subfig}
\usepackage[font=footnotesize]{caption}
\usepackage{bm}
\usepackage{mathrsfs} 
\usepackage{subfig}
\usepackage{comment}

\includecomment{longversion} 
\excludecomment{shortversion} 

\includecomment{onecol}
\excludecomment{twocol}

\usepackage[T1]{fontenc}

\usepackage{algorithm}
\usepackage{algorithmic}

\newcommand{\pval}[1]{\widehat{\level}(#1)}
\newcommand{\coeff}{\mu^2}

\newcommand{\cCorMat}{\Ac \MCov \Ac^\transp}

\newcommand{\algo}{CENTREx}
\newcommand{\dalgo}{DeCENTREx}
\newcommand{\weight}{\pastorv{w_{\matcov}}}

\def\invMDec{\Upsilonmat}

\newcommand{\pastorv}[1]{\textcolor{black}{#1}}

\newcommand{\duprazv}[1]{\textcolor{black}{#1}}

\newcommand{\NewStd}{\pastorv{r}}
\newcommand{\TheThreshold}[1]{\mu_{#1}}
\newcommand{\Topt}{\test_{\matcov}}

\newcommand{\thenorm}{\nu_{C}}

\newcommand{\TheFunc}{h_\matcov}
\def\myvspace{\vspace{0cm}}

\newcommand{\Zvec}{\Zbm}

\def\d{\mathrm{d}}
\newcommand{\VZ}[1]{\Zvec(#1)}
\newcommand{\VZctrd}{\Zvec^*}
\newcommand{\VW}[1]{\Wvec(#1)}

\def\dev{\Deltabm}
\def\myvspace{\vspace{0.2cm}}

\newcommand{\estCtr}{\widehat{\CCtr}}
\newcommand{\dt}{{\mathrm d} t}

\newcommand{\dz}{{\mathrm d} z}

\newcommand{\Noiseb}{\Xibm}
\newcommand{\noiseb}{\Xi}
\def\zetavec{\zetabm}
\def\Xivec{\Xibm}
\def\xivec{\xibm}
\def\Xvec{\Xbm}
\def\xvec{\xbm}

\def\Wvec{\Wbm}

\def\Ibf{\mathbf{I}}

\newcommand{\Rmat}{{\Rbm}}

\newcommand{\deltamat}{{\deltabm}}
\newcommand{\Upsilonmat}{\Upsilonbm}
\newcommand{\Abm}{{\bm A}}

\newcommand{\Cbm}{{\bm C}}

\newcommand{\Rbm}{{\bm R}}
\newcommand{\Upsilonbm}{{\bm \Upsilon}}

\newcommand{\Deltabm}{{\bm \Delta}}
\newcommand{\deltabm}{{\bm \delta}}

\newcommand{\Xibm}{{\bm \Xi}}
\newcommand{\zetabm}{{\bm \zeta}}
\newcommand{\xibm}{{\bm \xi}}
\newcommand{\Xbm}{{\bm X}}

\newcommand{\xbm}{{\bm x}}
\newcommand{\Ybm}{{\bm Y}}
\newcommand{\ybm}{{\bm y}}

\newcommand{\Wbm}{{\bm W}}

\newcommand{\Zbm}{{\bm Z}}

\newcommand{\fbm}{{\bm f}}

\def\Ebb{\mathbb{E}}

\def\Pbb{\mathbb{P}}

\def\Rbb{\mathbb{R}}

\def\Hcal{\mathcal{H}}

\def\Ncal{\mathcal{N}}
\def\Qcal{\mathcal Q}

\def\Scal{\mathcal{S}}

\def\Zcal{\mathcal{Z}}

\newcommand{\Tfrak}{\mathfrak T}

\def\ANcal{\mathcal{AN}}

\def\Rset{\Rbb}

\def\dim{d}
\def\dimc{m}
\newcommand{\Exp}[1]{\Ebb \! \left [ #1 \right ]}
\newcommand{\Expect}[1]{\Ebb \! \left [ #1 \right ]}
\newcommand{\tol}{\tau}

\newcommand{\transpose}{\mathrm{T}}
\newcommand{\transp}{\mathrm{T}}

\newcommand{\Noise}{\Xbm}

\newcommand{\matcov}{\Cbm}

\newcommand{\obs}{\Ybm}
\newcommand{\Obs}{\Ybm}
\newcommand{\cobs}{\Zbm}
\newcommand{\cObs}{\Zbm}

\newcommand{\mybig}{\big}

\newcommand{\test}{\Tfrak}

\def\radius{\rho}
\def\level{\alpha}

\newcommand{\Identity}[1]{\Ibf_{#1}}

\newcommand{\Ctr}{\bm \varphi}

\newcommand{\Marcum}{Q_{m/2}}
\newcommand{\Marcumc}{Q_{\dimc/2}}
\newcommand{\MyMarcum}{\Qcal}
\newcommand{\Id}{\Identity{\dim}}
\newcommand{\Idm}{\Identity{m}}

\newcommand{\Mk}{\mathcal{M}}
\newcommand{\Ectr}{\Phi}
\newcommand{\Ck}{\mathcal{C}_k}
\newcommand{\Ac}{\Abm}
\newcommand{\Nc}{\bm C}

\newcommand{\CCtr}{\bm{\phi}}

\newcommand{\MMDec}{\bm \Psi}
\newcommand{\MDec}{\textcolor{black}{\bm \Phi}}
\newcommand{\MCov}{\bm \Sigma}
\newcommand{\PhiMat}{\bm \Theta}
\newcommand{\Cphi}{\bm \theta}
\newcommand{\Nbesk}{V}

\newtheorem{prop}{Proposition}
\newtheorem{Lemma}{Lemma}

\newtheorem{Claim}{Claim}

\title{Decentralized Clustering on Compressed Data without Prior Knowledge of the Number of Clusters}
\author{Elsa Dupraz, Dominique Pastor, François-Xavier Socheleau \\ \small IMT Atlantique, Lab-STICC, Univ. Bretagne Loire, Brest, France 
	\thanks{A preliminary of this paper was published in the proceedings of ICASSP 2018 \cite{Dupraz_ICASSP2018}. This preliminary version contained part of the theoretical analysis without proof and the description of the centralized algorithm.}}

\begin{document}

\maketitle

\begin{abstract}
 In sensor networks, it is not always practical to set up a fusion center. Therefore, there is need for fully decentralized clustering algorithms. Decentralized clustering algorithms should minimize the amount of data exchanged between sensors in order to reduce sensor energy consumption. In this respect,  we propose one centralized and one decentralized clustering algorithm that work on compressed data without prior knowledge of the number of clusters. In the standard K-means clustering algorithm, the number of clusters is estimated by repeating the algorithm several times, which dramatically increases the amount of exchanged data, while our algorithm can estimate this number in one run.

 The proposed clustering algorithms derive from a theoretical framework establishing that, under asymptotic conditions, the cluster centroids are the only fixed-point of a cost function we introduce. This cost function depends on a weight function which we choose as the p-value of a Wald hypothesis test. This p-value measures the plausibility that a given measurement vector belongs to a given cluster. Experimental results show that our two algorithms are competitive in terms of clustering performance with respect to K-means and DB-Scan, while lowering by a factor at least $2$ the amount of data exchanged between sensors.
\end{abstract}

\section{Introduction}\label{sec:intro}
Wireless sensor networks are now used in a wide range of applications in medicine, telecommunications, and environmental domains, see~\cite{yick2008wireless} for a survey. For instance, they are employed for human health monitoring~\cite{omeni08BCS}, activity recognition on home environments~\cite{ordonez13S}, spectrum sensing in cognitive radio~\cite{sahasranand15ICC}, and so forth. In most applications, the network is asked to perform a given estimation, detection or learning task over measurements collected by sensors. 
In this paper, we consider clustering as a particular learning task. The purpose of clustering is to divide data into clusters such that the data inside a given cluster are similar with each other and different from the data belonging to the other clusters~\cite{jain10PRL}.

In this paper, we would like to take the following two major practical constraints into account. 
First, sensor networks usually require a fusion center, whose twofold role is to receive the data and achieve the desired learning task. In this case, we say that the learning task is centralized. 
However, it is not always practical to set up a fusion center, especially in recent applications involving autonomous drones or robots~\cite{tuna14AHN}. In such applications, the sensors should perform the learning task by themselves in a fully decentralized setup, without resorting to any fusion center. In this case, we say that the learning task performed by the network is decentralized. 

Second,  
it is crucial to reduce the sensors energy consumption in order to increase the network lifetime. Since most of the energy of the sensors is consumed by transmitting data via the communication system, it is highly desirable to transmit these data in a compressed form so as to lower the energy consumption. In addition, because the objective of a clustering task is not to reconstruct all the sensor measurements but only to cluster them, performing this task on compressed data is all the more desirable as it avoids costly decoding operations.

According to the foregoing, our focus is thus on the design of clustering algorithms that work directly over compressed measurements in a fully decentralized setup. In such a decentralized setup, each sensor must perform the clustering with only partial observations of the available compressed measurements, while minimizing the amount of data exchanged in the network. 

Clustering over compressed measurements was recently addressed in~\cite{boutsidis15IT,keriven2017compressive,dupraz2018k}, for the K-means algorithm only. The K-means algorithm is very popular due to its simplicity and effectiveness \cite{jain10PRL}. It makes no assumption on the signal model of the measurement vectors that belong to a cluster and, as such, it is especially relevant for applications such as document classification~\cite{steinbach2000comparison}, information retrieval, or categorical data clustering~\cite{huang1998extensions}. However, the K-means algorithm requires prior knowledge of the number $K$ of clusters, which is not acceptable in a network of sensors where the data is non-stationary and $K$ may vary from one data collection to another. When $K$ is unknown, one could think of applying a penalized method~\cite{pelleg2000x} that permits to jointly estimate $K$ and perform the clustering. Unfortunately, this method requires running the K-means algorithm several times with different numbers of clusters, which may be quite energy consuming. 
As another issue, the K-means algorithm must be initialized properly in order to get a chance to correctly retrieve the clusters. Proper initialization can be obtained with the K-means++ procedure~\cite{arthur07SIAM}, which requires computing all the two-by-two distances between all the measurement vectors of the dataset. As a result, the K-means++ procedure is not affordable in a decentralized setup.
It is worth mentioning that the variants of K-means such as Fuzzy K-means~\cite{wu02PR} suffer from the same two issues. 

Other clustering algorithms such as DB-SCAN~\cite{ester1996density} and OPTICS~\cite{ankerst1999optics} may appear as suitable candidates for decentralized clustering since they do not need the number of clusters. However, they require setting two parameters that are the maximum distance between two points in a cluster and the minimum number of points per cluster. These parameters have a strong influence on the clustering performance, but they can hardly be estimated and they must be chosen empirically~\cite{Gan2015DBSCANRM}.

Therefore, our purpose is to derive a solution that bypasses the aforementioned issues for clustering compressed data in a decentralized setup. In this respect, we proceed in two main steps. We begin by introducing a centralized clustering algorithm that circumvents the drawbacks of the standard algorithms. This algorithm is hereafter named \algo~as it performs the clustering without prior knowledge of the number of clusters. In a second step, we devise a decentralized version \dalgo~of this algorithm. 

Crucially, \algo~derives from a model-based theoretical approach. We hereafter consider the same Gaussian model as in~\cite{bouveyron14CSDA, forero11SP}. In this model recalled in Section~\ref{sec:model}, the measurement vectors belonging to a given cluster are supposed to be the cluster centroid corrupted by additive Gaussian noise. 
In our model, the Gaussian noise is not necessarily independent and identically distributed (i.i.d.) as it is described by a non-diagonal covariance matrix. 
Here, in contrast to~\cite{bouveyron14CSDA, forero11SP}, we will not assume a known number of clusters, but suppose that the covariance matrix is known. This assumption was already made for clustering in~\cite{wu02PR,zhou16ICC} in order to choose the parameters for the functions that compute the cluster centroids. Further, in a sensor network context, this assumption is more acceptable than a prior known number of clusters. Indeed, in many signal processing applications, the noise covariance matrix can be estimated, either on-the-fly or from preliminary measurements, via many parametric, non-parametric, and robust methods (see \cite{Huber2009, Rousseeuw93, Pastor2012a}, among others).

On the basis of this model, a new cost function for clustering over compressed data is introduced in Section~\ref{sec:robust}. This cost function generalizes the function introduced in~\cite{wu02PR} for clustering over non-compressed data.
In~\cite{wu02PR}, the choice of the cost function was justified by an analogy with M-estimation~\cite{zoubir12robust}, but it was not supported by any theoretical arguments related to clustering.
On the opposite, the novel theoretical analysis we conduct in Section~\ref{sec:robust} shows that, under asymptotic conditions, the compressed cluster centroids are the only minimizers of the introduced cost function. 
The cost function depends on a weight function that must verify some properties deriving from the theoretical analysis. 
As exposed in Section~\ref{sec:w_function}, the weight function is chosen as the p-value of a Wald hypothesis test~\cite{Wald1943}. This p-value measures the plausibility that a measurement vector belongs to a given cluster. 
In addition, its expression does not depend on any empirical parameter that could influence the final clustering performance. 

In Sections~\ref{sec:algo} and~\ref{sec:dalgo}, we describe the clustering algorithms \algo~and~\dalgo~that derive from our mathematical analysis.
Given the compressed measurements, both algorithms estimate the compressed cluster centroids one after each other by computing the minimizers of our cost function, even when the number of minimizers is \emph{a priori} unknown. 
The clustering is then performed by assigning each measurement to the cluster with the closest estimated centroid. 
The decentralized version~\dalgo~takes advantage of the fact that our approach does not require prior knowledge of the number of clusters, and that it does not suffer from initialization issues.
We show that, due to these advantages, the amount of data to be exchanged between sensors for~\dalgo~is much lower than for decentralized K-means~\cite{datta09KDE}.
%
Simulation results presented in Section~\ref{sec:experiments} show that our algorithms give much better performance than DB-Scan and that they only suffer a small loss in performance compared to K-means with known $K$. We also observe that our algorithms give the same level of clustering performance as K-means with $K$ \emph{a priori} unknown, while requiring less data exchange.  


\section{Signal model and notation}\label{sec:model}
In this section, we introduce our notation and assumptions for the signal model and the data collected by sensors in the network. We also recall the definition of the Mahalanobis norm which will be useful in the theoretical analysis proposed in the paper. 
\vspace{-0.25cm}
\subsection{Signal Model}
In this paper, the notation $\llbracket 1,N\rrbracket$ denotes the set of integers between $1$ and $N$.
Consider a set of $N$ independent and identically distributed (i.i.d.) $\dim$-dimensional random Gaussian vectors $\obs_1, \ldots, \obs_N$ with same covariance matrix $\MCov$.
We consider that the $N$ measurement vectors are split into $K$ clusters defined by $K$ \duprazv{deterministic} centroids $\Ctr_1, \ldots, \Ctr_K$, with $\Ctr_k \in \Rset^\dim$ for each $k \in \llbracket 1, K \rrbracket$. 
Accordingly, we assume that for each $n \in \llbracket 1, N \rrbracket$, there exists $k \in \llbracket 1, K \rrbracket$ such that $\Obs_n \thicksim \Ncal(\Ctr_k, \MCov)$ and we say that $\obs_n$ belongs to cluster $k$.
In the following, we assume that the covariance matrix $\MCov$ is known prior to clustering. 

The measurement vectors are all multiplied by a sensing matrix $\Ac \in \mathbb{R}^{m\times d}$, which produces compressed vectors $\cobs_n = \Ac \obs_n$, $n \in \llbracket 1, N \rrbracket$.
As a result, $\cobs_n \sim \mathcal{N}(\CCtr_k, \Ac \MCov \Ac^\transp )$, where $\CCtr_k = \Ac \Ctr_k$ represents the compressed centroids.
Here, the matrix $\Ac$ is known and it is the same for all the sensors. \pastorv{It is assumed to have full rank so that $\Ac^\transp$ is injective}. This matrix performs compression whenever $m < d$. The theoretical analysis presented in the paper applies whatever the considered full rank matrix, and in our simulations, we will consider several different choices for $\Ac$.

In the paper, the data repartition in the network will depend on the considered setup. In the centralized setup, we will assume that all the compressed vectors $\cobs_1,\cdots, \cobs_N$ are available at a fusion center. In the decentralized setup, we will assume that the network is composed by $S$ sensors which all observe a different subset of the measurement vectors.

In the following, we start by describing the centralized version of the algorithm. 
We assume that the centroids $\Ctr_1, \ldots, \Ctr_K$, and their compressed versions $\CCtr_1, \ldots, \CCtr_K$, are unknown. 
We want to propose an algorithm that groups the $N$ compressed measurement vectors $\cobs_1,\cdots, \cobs_N$ into clusters, without prior knowledge of the number of clusters.
The first step of our algorithm consists of estimating the compressed centroids $\CCtr_1, \ldots, \CCtr_K$. Our centroid estimation method relies on the Mahalanobis norm whose properties are recalled now.

\vspace{-0.1cm}
\subsection{Mahalanobis norm} 
Consider an $\dimc \times \dimc$ positive-definite matrix $\matcov$. The Mahalanobis norm $\nu_\matcov$ is defined for any $\xvec \in \Rset^\dimc$ by setting $\nu_{\Nc}(\xvec) = \sqrt{\xvec^T \Nc^{-1} \xvec}$. If $\matcov$ is the identity matrix $\Identity{\dimc}$, the Mahalanobis norm $\nu_\matcov$ is the standard Euclidean norm $\| \cdot \|$ in $\Rset^\dimc$. More generally, since $\Nc$ is positive definite, it can be decomposed as $\pastorv{\Nc = \Rmat \deltamat \Rmat^{\transp}},$ 
where \pastorv{$\deltamat$} is a diagonal matrix \pastorv{whose diagonal values are} the eigenvalues of $\Nc$ and \pastorv{$\Rmat$} contains the corresponding eigenvectors.
By setting $\pastorv{\MMDec = \deltamat^{-1/2}\Rmat^\transp}$ 
it is easy to verify that:
\begin{equation}\label{Eq: Whitening matrix}
\MMDec \Nc \MMDec^{\transp} = \Idm \quad \text{and} \quad \nu_{\Nc}(\xvec) = \| \pastorv{\MMDec} \xvec \| \, (\xvec \in \Rset^\dimc).
\end{equation}
According to~\eqref{Eq: Whitening matrix}, $\MMDec$ is called the {\em whitening matrix} of $\matcov$.

\section{Centroid estimation}
\label{sec:robust}
In this section, we introduce a new cost function for the estimation of the compressed centroids $\CCtr_1,\cdots, \CCtr_K$ from the measurement vectors $\cobs_1,\cdots,\cobs_n$.
We then present our theoretical analysis that shows that the compressed centroids $\CCtr_k$ are the only minimizers of the cost function. 

\vspace{-0.25cm}
\subsection{Cost Function for centroid estimation}
\label{subsec:cost function}
Consider an increasing, convex and differentiable function $\rho: \mathbb{R} \rightarrow \mathbb{R}$ that verifies $\rho(x)=0 \Rightarrow x=0$. First assume that the number $K$ of clusters is known, and consider the following cost function for the estimation of the compressed centroids:
\pastorv{
 \vspace{-0.1cm}
 \begin{equation}\label{eq:cost_function}
  J(\PhiMat) = \sum_{k=1}^K \sum_{n=1}^N \rho(\nu_{\Nc}^2( \cObs_n - \Cphi_k))
 \end{equation}
 with $\PhiMat = (\Cphi_1, \ldots, \Cphi_K)$.
 }
 This cost function generalizes the one introduced in~\cite{wu02PR} for centroid estimation when $K$ is known.
 In~\cite{wu02PR}, the clustering was performed over \pastorv{i.i.d.} Gaussian vectors, and the particular case $\Nc=\Id$ was considered.
\pastorv{In contrast}, our analysis assumes a general positive-definite matrix $\Nc$, which will permit to take into account both a non-diagonal covariance matrix $\MCov$ and the correlation introduced by the compression matrix $\Ac$.
 In addition,~\cite{wu02PR} only considers \pastorv{the} particular case $\rho(x) = 1-\exp(-\beta x)$, where $\beta$ is a parameter that has to be chosen empirically. 
 On the opposite, here, we consider a class of possible functions $\rho$, and the properties that these functions should verify will be \pastorv{exposed} in the subsequent theoretical analysis. 
 Note that the approach in~\cite{wu02PR} was inspired by the M-estimation theory~\cite{zoubir12robust}.
 
 In order to estimate the centroids, we want to minimize the cost function~\eqref{eq:cost_function} with respect to \pastorv{$\PhiMat$}.
 Since $J$ is convex \pastorv{by the properties of $\nu_\matcov$ and $\rho$,} $\rho$ is differentiable, and $\matcov$ is invertible, \pastorv{standard matrix differentiation~\cite[Sec. 2.4]{Matrixcookbook} allow to show that}  
 the minimizer $\PhiMat$ of~\eqref{eq:cost_function} should verify
\pastorv{
\vspace{-0.1cm}
	\begin{equation}\label{eq:cancel_J}
\forall k \in \llbracket 1, K \rrbracket, ~ \sum_{n=1}^N (\cObs_n-\Cphi_k ) \weight(\cObs_n - \Cphi_k) = 0
\end{equation}
where $\weight = w \circ \nu^2_\matcov$ is hereafter called the $\dimc$-dimensional weight function and $w = \rho'$} 
is called the \pastorv{scalar weight function}. Unless necessary, we generally drop the adjective 'scalar' in the sequel.

Solving~\eqref{eq:cancel_J} amounts to looking for the fixed-points $\TheFunc(\CCtr) = \CCtr$ of the function $\TheFunc$ defined as 
	\begin{equation}
	\label{eq:h}
	\TheFunc(\CCtr) = \dfrac{\sum_{n = 1}^{N} \weight ( \cObs_n-\CCtr)  \cObs_n}{\sum_{n = 1}^{N} \weight (\cObs_n-\CCtr)}, \CCtr \in \Rset^\dimc.
	\end{equation}
In~\cite{wu02PR}, no theoretical argument was given to demonstrate that the introduced cost function was appropriate for the estimation of the cluster centroids.
On the opposite, in the following, we show the following strong result: the centroids $\CCtr_k$ are the only fixed points of $\TheFunc$ under asymptotic conditions, provided that the weight function $w$ verifies certain properties. 

Perhaps surprisingly, the expression of $\TheFunc$ depends \pastorv{neither} on the considered cluster $k$, nor on the number of clusters $K$. The foregoing suggests that, even when $K$ is unknown, estimating the centroids can be performed by seeking the fixed points of $\TheFunc$. This claim is theoretically and experimentally verified below for a certain class of matrices $\matcov$.

\vspace{-0.25cm}
\subsection{Fixed-point analysis}
The following proposition shows that the compressed centroids $\CCtr_k$ are the only fixed points of the function $	\TheFunc$ defined in~\eqref{eq:h}. 

\begin{prop}
\label{Prop: fixed points}
With the same notation as above, let $N_k$ be the number of data belonging to cluster $k \in \llbracket 1, K \rrbracket$ and set $N = \sum_{k=1}^K N_k$. Assume that there exist $\alpha_1, \ldots, \alpha_K \in (0,1)$ such that $\lim\limits_{N \to \infty} N_k/N = \alpha_k$. Assume also that the function $w$ is non-null, non-negative, continuous, bounded and verifies:
\begin{align}
\displaystyle \lim\limits_{\duprazv{t \rightarrow \infty}} {w ( t )} & = 0 \label{eq:prop3} .
\end{align}
For any positive definite matrix  $\matcov$ proportional to $\cCorMat$, for any $i \in \llbracket 1, K \rrbracket$, and any $\varepsilon > 0$:
\vspace{-0.25cm}
\begin{onecol}
\begin{equation}\nonumber
\label{Eq: Asymptotic behavior of gN(theta)-theta-2}
\left \{ \CCtr \in \Rset^m: 
\displaystyle \lim_{\forall k \ne i, \| \CCtr_k - \CCtr_i \| \rightarrow \infty}
\left ( \, 
\displaystyle \lim_{N \to \infty}
\big ( \, \TheFunc(\CCtr) - \CCtr \, \big ) 
\, \right ) = 0 
\, \, \& \, \, \| \CCtr - \CCtr_i \| \leqslant \varepsilon \right \} = \Big \{ \CCtr_i \Big \} \quad \text{(a-s)}
\end{equation}
\end{onecol}

\begin{twocol}
	\vspace{-0.1cm}
	\begin{align}\nonumber
\label{Eq: Asymptotic behavior of gN(theta)-theta-2}
 \Big\{ \CCtr \in \Rset^m: &
\displaystyle \lim_{\forall k \ne i, \| \CCtr_k - \CCtr_i \| \rightarrow \infty}
 ( \, 
\displaystyle \lim_{N \to \infty}
\big ( \, \TheFunc(\CCtr) - \CCtr \, \big ) 
\,  ) = 0  \\
& \, \, \& \, \, \| \CCtr - \CCtr_i \| \leqslant \varepsilon \Big\} = \Big \{ \CCtr_i \Big \} \quad \text{(a-s)}
	\end{align}
\end{twocol}
\end{prop}

\begin{IEEEproof}
For any $k \in \llbracket 1, K \rrbracket$, let $\cObs_{k,1}, \ldots, \cObs_{k,N_k}$ be the $N_k$ compressed vectors that belong to cluster $k$. 
Consider a given matrix $\matcov$ that is positive definite and proportional to $\cCorMat$. We can write $\TheFunc(\CCtr)$ in the form:
	\begin{equation}
	\label{eq:h2}
\pastorv{
	\TheFunc(\CCtr) = \dfrac{\displaystyle \sum_{k=1}^K \sum_{n = 1}^{N_k} \weight ( \cObs_{k,n}-\CCtr )  \cObs_{k,n}}{\displaystyle \sum_{k=1}^K \sum_{n = 1}^{N_k} \weight ( \cObs_{k,n} - \CCtr ) }, \CCtr \in \Rset^\dimc.
}
	\end{equation} 
The random function (\ref{eq:h2}) can then be rewritten as
$\TheFunc(\CCtr) =  {U_N(\CCtr)}/{V_N(\CCtr)}$
with:
	\vspace{-0.1cm}
\begin{equation}
\label{Eq: U et V}
\left \{
\begin{array}{lll}
U_N(\CCtr) = \displaystyle \sum_{k=1}^K \sum_{n = 1}^{N_k} \weight ( \cObs_{k,n}-\CCtr )  \cObs_{k,n} \myvspace 
\\
V_N(\CCtr) = \displaystyle \sum_{k=1}^K \sum_{n = 1}^{N_k} \weight ( \cObs_{k,n}-\CCtr ).
\end{array}
\right.
\end{equation}
Therefore, $\TheFunc(\CCtr) - \CCtr = W_N(\CCtr)/V_N(\CCtr)$, with $W_N(\CCtr) = U_N(\CCtr) - V_N(\CCtr) \CCtr$. 

For any $(k,n) \in \llbracket 1,K \rrbracket \times \llbracket 1, N_k \rrbracket$, we set $\dev_k = \CCtr_k - \CCtr$, $\alpha_{k,N} = N_k/N$ and $\Noise_{k,n} = \cObs_{k,n} - \CCtr_k$. With this notation, we have $\Noise_{k,1}, \ldots, \Noise_{k,N_k} \stackrel{\text{iid}}{\thicksim} \Ncal(0,\cCorMat)$ as well as:
\begin{equation}
\label{Eq: WN/N}
\dfrac{1}{N} W_N(\CCtr) = \displaystyle \sum_{k=1}^K \alpha_{k,N} \frac{1}{N_k}\displaystyle \sum_{n=1}^{N_k} \weight ( \dev_k + \Noise_{k,n} ) \left ( \dev_k + \Noise_{k,n} \right )
\end{equation}
and:
\vspace{-0.1cm}
\begin{equation}
\label{Eq: VN/N}
\dfrac{1}{N} V_N(\CCtr) = \displaystyle \sum_{k=1}^K \alpha_{k,N} \frac{1}{N_k}\displaystyle \sum_{n=1}^{N_k} \weight ( \dev_k + \Noise_{k,n} ) .
\end{equation}
By the strong law of large numbers, it follows from \eqref{Eq: WN/N} and \eqref{Eq: VN/N} that for any $i \in \llbracket 1,K\rrbracket$, 
\begin{onecol}
\begin{equation}\nonumber
\label{Eq: Asymptotic behavior of gN(theta)-theta-2}
\displaystyle \lim_{N \to \infty} \left (\TheFunc(\CCtr) - \CCtr \right ) = 
\dfrac{\alpha_{i,N} \, \Expect{\weight ( \VZ{\dev_i} ) \VZ{\dev_i} } + \displaystyle \sum_{k=1, k \ne i}^K \alpha_{k,N} \, \Expect{ \weight ( \VZ{\dev_k} ) \VZ{\dev_k} }}{\alpha_{i,N} \, \Expect{ \weight ( \VZ{\dev_i} ) } + \displaystyle \sum_{k=1, k \ne i}^K \alpha_{k,N} \, \Expect{ \weight ( \VZ{\dev_k})}} \quad \text{(a-s)}
\end{equation}
\end{onecol}
\begin{twocol}
	\begin{equation}\nonumber
	\label{Eq: Asymptotic behavior of gN(theta)-theta-2}
	\displaystyle \lim_{N \to \infty} \left (\TheFunc(\CCtr) - \CCtr \right ) = 
	\dfrac{\alpha_{i,N} \, \Expect{\weight ( \VZ{\dev_i} ) \VZ{\dev_i} } + S_1}{\alpha_{i,N} \, \Expect{ \weight ( \VZ{\dev_i} ) } + S_2} \quad \text{(a-s)}
	\end{equation}
\begin{equation}
\nonumber
\text{with} \, \, 
\left \{
\begin{array}{lll}
S_1 & = & \displaystyle \sum_{k=1, k \ne i}^K \alpha_{k,N} \, \Expect{ \weight ( \VZ{\dev_k} ) \VZ{\dev_k} }, \\
S_2 & = & \displaystyle \sum_{k=1, k \ne i}^K \alpha_{k,N} \, \Expect{ \weight ( \VZ{\dev_k})}
\end{array}
\right.
\end{equation}
\end{twocol}
and $\VZ{\dev_k} \thicksim \Ncal(\dev_k,\Ac\MCov \Ac^\transp)$.
Assume now that $\| \dev_i \| \leqslant \varepsilon$. It follows from \eqref{eq:prop3} and Lemma \ref{Lemma: asymptotic behaviors} of Appendix \ref{App: asymptotic behaviors} that: 
\begingroup\small
$$
\pastorv{
\displaystyle \lim_{\forall k \ne i, \| \CCtr_k \! - \CCtr_i \| \rightarrow \infty} \! \!
\left ( 
\displaystyle \lim_{N \to \infty} \! 
\big ( \TheFunc(\CCtr) \! - \! \CCtr \big ) 
\right ) \! \! = \! \! 
\dfrac{\Expect{\weight ( \VZ{\dev_i} ) \VZ{\dev_i} }}{\Expect{ \weight ( \VZ{\dev_i} ) }} \text{(a-s)}
}
\vspace{0.25cm}
$$ \endgroup
Because $\Expect{ \weight ( \VZ{\dev_i} ) } = \Expect{w ( \nu_{\Nc}^2( \VZ{\dev_i} ) ) } > 0$, the left hand side (lhs) to the equality above is $0$ if and only if $\Expect{\weight( \VZ{\dev_i} )  \VZ{\dev_i} } = 0.$ 

\pastorv{Let $\MMDec$ be the whitening matrix of $\matcov$. From \eqref{Eq: Whitening matrix}, we get that:}
$$\Expect{\weight( \VZ{\dev_i} ) ) \VZ{\dev_i} } = \pastorv{\MMDec^{-1}\Expect{w ( \| \MMDec \VZ{\dev_i}\|^2 ) \MMDec \VZ{\dev_i}}}.$$
\pastorv{For $\coeff \ne 0$ such that $\cCorMat = \coeff \matcov$,~\eqref{Eq: Whitening matrix} also induces that
\begin{equation}
\label{Eq: decomposition of cCormat}
\MMDec \cCorMat \MMDec^\transp = \coeff \Idm.
\end{equation}
Therefore, $\MMDec \VZ{\dev_i} \! \! \sim \! \! \mathcal{N}(\MMDec \dev_i, \coeff \Idm)$}. \pastorv{By Lemma~\ref{Lemma: last useful lemma} of Appendix~\ref{App:D} and the properties of $\MMDec$, we conclude that $\Expect{w ( \| \MMDec \VZ{\dev_i}\|^2 ) \MMDec \VZ{\dev_i} } \! = \! 0$ if and only if $\dev_i \! = \! 0$.}
\end{IEEEproof}

Proposition~\ref{Prop: fixed points} shows that the centroids are the unique fixed points of the function $\TheFunc$, when the sample size $N$ and the distances between centroids tend to infinity. 
\pastorv{This result means that,} at least asymptotically, no vector other than a centroid can be a fixed point of $\TheFunc$.
This result, as well as the fact that the expression of $\TheFunc$ depends on neither $k$ nor $K$, will allow us to derive a clustering algorithm that does not require prior knowledge of $K$.

We however wonder about the statistical behavior of the fixed points of $\TheFunc$ in non-asymptotic situations.
In particular, the non-asymptotic fixed-points statistical model derived in the next section will help us refine our clustering algorithm. Although derived from some approximations, this model will allow us to choose weight functions $w$ that verify the conditions of Proposition~\ref{Prop: fixed points} and that are also suitable when the sample size and the distances between centroids are finite. 

\vspace{-0.25cm}
\subsection{Fixed point statistical model}
\label{subsec: fixed point statistical model}

Under the assumptions of Proposition \ref{Prop: fixed points}, a fixed point of $\TheFunc$ provides us with an estimated centroid $\estCtr_k$ for some unknown centroid $\CCtr_k$. 
The following claim gives the statistical model we consider for the estimated centroids $\estCtr_k$. This result is given by a claim rather than a proposition, since its derivation is based on several approximations. 
\duprazv{
\begin{Claim}\label{claim:model_fp}
\pastorv{For any positive definite matrix  $\matcov = (1/\coeff) \cCorMat$ with $\coeff \neq 0$ and all $k \in \llbracket 1,k \rrbracket$, we approximate the statistical model of $\estCtr_k$ as 
\begin{equation}
\label{Eq: model of estimated centroids}
\estCtr_k \thicksim \Ncal(\CCtr_k, (\NewStd^2/N_k) \, \matcov ), 
\end{equation}
where $N_k$ is the number of compressed vectors in cluster $k$ and
\begin{equation}
\label{eq:rho}
\NewStd^2 = \dfrac{\Expect{w^2 \left ( \| \Noiseb \|^2 \right ) \noiseb_1^2}}{\Expect{w \left ( \| \Noiseb \|^2 \right )}^2}
\end{equation}
with $\Noiseb = (\noiseb_1, \ldots, \noiseb_\dimc)^\transp \thicksim \Ncal(0, \coeff \Idm)$.}
\end{Claim}}

\subsubsection*{Derivation}
In order to model the estimation error, we can start by writing $\TheFunc(\estCtr_k) = \TheFunc(\CCtr_k) + W_{k,1}$. Of course, $W_{k,1}$ will be all the more small than $\estCtr_k$ approximates accurately $\CCtr_k$. We can then write that $\TheFunc(\CCtr_k) = g_k(\CCtr_k) + W_{k,2}$, where 
\begin{equation}
\label{Eq:gk}
g_k(\xbm) = \dfrac{\sum_{n = 1}^{N_k} \weight ( \cobs_{k,n}-\xbm )  \cobs_{k,n}}{\sum_{n = 1}^{N_k} \weight ( \cobs_{k,n}-\xbm ) }, \xbm \in \Rset^\dimc
\end{equation}
The term $W_{k,2}$ is likely to be small if $\forall k' \neq k$, $\forall n \in \llbracket 1,N_{k'}\rrbracket $, $\weight(\cobs_{k',n} - \CCtr_k) \ll 1$, that is if the function $\weight$ reduces strongly \pastorv{the influence of} data from clusters other than $k$.
To finish, in absence of noise, we would directly have $g_k(\CCtr_k) = \CCtr_k$, but the presence of noise induces that $g_k(\CCtr_k) = \CCtr_k + W_{k,3}$. Finally, we have $\estCtr_k  = \CCtr_k + W_{k,3} + W_{k,2} + W_{k,1}$. 

We now derive a model for $W_{k,3}$, and we keep the same notation as in the proof of Proposition \ref{Prop: fixed points}. In particular, $\Noise_{k,1}, \ldots, \Noise_{k,N_k} \stackrel{\text{iid}}{\thicksim} \Ncal(0,\cCorMat)$ with $\Noise_{k,n} = \cObs_{k,n} - \CCtr_k$ for any $k \in \llbracket 1, K\rrbracket$ and any $n \in \llbracket 1, N_k\rrbracket$. It follows from (\ref{Eq:gk}) that $W_{k,3} = g_k(\CCtr_k) - \CCtr_k = S_{N_k}/T_{N_k}$
	with $S_{N_k} = \sum_{n = 1}^{N_k} \weight ( \Noise_{k,n} ) \Noise_{k,n}$ and 
	$T_{N_k} = \sum_{n = 1}^{N_k} \weight ( \Noise_{k,n})$. 
	The random variables $\weight ( \Noise_{k,n} ) \Noise_{k,n}$ are iid and we proceed by computing their mean and covariance matrix.

	Given any $\Noise \thicksim \Ncal(0,\cCorMat)$, $\Exp{ \weight ( \Noise_{k,n} ) \Noise_{k,n}} = \Exp{ \weight ( \Noise ) \Noise}$ for any $n \in \llbracket 1, N_k \rrbracket$. As above, let $\MMDec$ be the whitening matrix of $\Nc$. According to \eqref{Eq: Whitening matrix}, $\Exp{ \weight ( \Noise ) \Noise} = \invMDec \, \Exp{ w \left ( \| \Xivec \|^2 \right ) \Xivec}$ where $\Xivec = \MMDec\Noise$ and $\invMDec = \MMDec^{-1}$. It then follows from \eqref{Eq: decomposition of cCormat} that $\Xivec \thicksim \Ncal(0,\coeff \Idm)$. We derive from the foregoing and Lemma \ref{Lemma: last useful lemma} of Appendix \ref{App:D} that $\Exp{ w\left ( \| \Xivec \|^2 \right ) \Xivec} = 0$ and thus, that $\Exp{ \weight ( \Noise ) \Noise} = 0$. 

	With the same notation as above, the covariance matrix of any $ \weight ( \Noise_{k,n} ) \Noise_{k,n}$ is that of $\weight ( \Noise)\Noise$. Since this random vector is centered, its covariance matrix equals 
	$$\Exp{\pastorv{w^2_\matcov}( \Noise )\Noise \Noise^\transp} = \invMDec \, \Exp{w^2( \| \Noiseb \|^2) \, \Noiseb \Noiseb^\transp)} \invMDec^\transp.$$
	Lemma \ref{Lemma: Covmat} of Appendix \ref{App: Correlation matrix} implies that $\Exp{\pastorv{w^2_\matcov}(\Noise)\Noise \Noise^\transp} = \Exp{w^2( \| \Noiseb \|^2) \, \noiseb_1^2} \! \matcov.$ By the central limit theorem, $S_{N_k}/\sqrt{N_k}$ thus converges in distribution to $\Ncal \left ( 0,\Exp{w^2( \| \Noiseb \|^2) \, \noiseb_1^2} \! \matcov \, \right )$.	
	By the weak law of large numbers and \eqref{Eq: Whitening matrix} again, $T_{N_k}/N_k$ converges in probability to $\Expect{w \! \left ( \| \Noiseb \|^2 \right )}$.
	Slutsky's theorem \cite[Sec. 1.5.4, p. 19]{Serfling1980} implies that $\sqrt{N_k} S_{N_k}/T_{N_k}$ converges in distribution to $\Ncal \left ( 0,\NewStd^2 \matcov \right )$,
	where $\NewStd^2$ is defined in~\eqref{eq:rho}.

	 Therefore, $W_{k,3}$ is asymptotically Gaussian so that $g_k(\CCtr_k) = \CCtr_k + W_{k,3} \thicksim \ANcal \left ( \CCtr_k, ({\NewStd^2}/{N_k}) \, \matcov \right )$. We do not know how to model $W_{k,1}$ and $W_{k,2}$ yet. We merely know that the contributions of these two types of noise are small under the asymptotic conditions of Proposition \ref{Prop: fixed points}. As a result, we do not take the influence of $W_{k,1}$ and $W_{k,2}$ into account and model the statistical behavior of $\estCtr_k$ by~\eqref{Eq: model of estimated centroids}.
	\qed

In the above model, $\NewStd^2$ can be calculated by Monte-Carlo simulations, and we will explain in the algorithm description how we estimate $N_k$. 
Although \eqref{Eq: model of estimated centroids} may be a coarse approximation, since $W_{k,1}$ and $W_{k,2}$ are not necessarily negligible compared to $W_{k,3}$, the experimental results reported in Section \ref{sec:experiments} support the practical relevance of the approach.

At the end, all the results of this section were derived from a generic \pastorv{scalar} weight function $w$, and the theoretical analysis provided the properties that $w$ should \pastorv{satisfy}.
In the following, we choose a weight function $w$ that \pastorv{satisfies} these properties and that is suitable for clustering.

\section{Weight function}
\label{sec:w_function}
The \pastorv{scalar} weight function $w(x)= \pastorv{\beta}\exp(-\beta x)$ proposed in~\cite{wu02PR} verifies the properties required in Proposition~\ref{Prop: fixed points}.
However, in this weight function, the parameter $\beta$ must be chosen empirically and its optimal value varies with $m$ and the noise parameters.
A poor choice of $\beta$ can dramatically impact the performance of the clustering algorithm proposed in~\cite{wu02PR}. 

In contrast, we propose new weight functions whose expressions are known whatever the dimension and noise parameters. \pastorv{These new weight functions are devised as the p-values of Wald's hypothesis tests for testing the mean of a Gaussian~\cite{Wald1943}}. In this section, we thus begin by recalling the basics about Wald's test for testing the mean of a Gaussian and we introduce the p-value for this test. We then derive the weight functions that will be used in our clustering algorithm. 

\vspace{-0.25cm}
\subsection{p-value of Wald's test for testing the mean of a Gaussian}
\label{subsec: Wald}
Let $\Xvec \thicksim \Ncal(\xivec,\matcov)$, where the $\dimc \times \dimc$ covariance matrix $\matcov$ is positive definite and $\xivec \in \Rset^\dimc$ is unknown. 
Consider the problem of testing whether $\Xvec$ is centered or not. This problem can be summarized as:
\begin{equation}
\label{Eq:ddtpb}
\left \{
\begin{array}{lll}
\text{{Observation:}} \, \Xvec \thicksim \Ncal(\xivec,\matcov), \\
\text{Hypotheses:} \,
\left \{
\begin{array}{lll}
\Hcal_0: \,  \xivec = 0, \\
\Hcal_1: \, \xivec \neq 0.
\end{array}
\right.
\end{array}
\right.
\end{equation}

Recall that a non-randomized test $\test$ is any measurable map from $\Rset^\dimc$ to $\{0,1\}$. Given \pastorv{a realization $\xvec \in \Rset^\dimc$ of $\Xvec$}, the value $\test(\pastorv{\xvec})$ returned by $\test$ is the index of the hypothesis considered to be true. We say that $\test$ accepts $\Hcal_0$ (resp. $\Hcal_1$) at $\pastorv{\xvec}$ if $\test(\pastorv{\xvec}) = 0$ (resp. $\test(\pastorv{\xvec}) = 1$).  \pastorv{Given $\level \in (0,1)$, let $\TheThreshold{\level}$ be the unique real value such that:
	\begin{equation}
	\label{Eq: Wald threshold}
	\Marcum(0,\TheThreshold{\level}) = \alpha,
	\end{equation}
where $\Marcumc$ is the Generalized Marcum Function~\cite{Sun2010}.} According to \cite[Definition III \& Proposition III, p. 450]{Wald1943}, the non-randomized test defined for any $\pastorv{\xvec} \in \Rset^\dimc$ as
\begin{equation}
\label{Eq:Thresholding test from above}
\Topt(\xvec) = \left \{
\begin{array}{lll}
0 & \hbox{ if } & \thenorm( \pastorv{\xvec} ) \leqslant  \TheThreshold{\level}\\
1 & \hbox{ if } & \thenorm( \pastorv{\xvec} ) > \TheThreshold{\level} .
\end{array}
\right.
\end{equation}
guarantees a false alarm probability $\level$ for the problem described by~\eqref{Eq:ddtpb}. Although there is no Uniformly Most Powerful (UMP) test for the composite binary hypothesis testing problem (\ref{Eq:ddtpb}) \cite[Sec. 3.7]{Lehmann2005}, $\Topt$ turns out to be optimal with respect to several optimality criteria and within several classes of tests with level $\level$ \cite[Proposition 2]{RDT}. In particular, $\Topt$ is UMP \pastorv{with size $\level$} among all spherically invariant tests and has Uniformly Best Constant Power (UBCP) on the spheres centered at the origin of $\Rset^\dimc$ \cite[Definition III \& Proposition III, p. 450]{Wald1943}. It is hereafter called a Wald test, without recalling explicitly the level $\pastorv{\level}$ at which the testing is performed. 

In Appendix \ref{Sec: RDT pvalue}, we show that test $\Topt$ has a p-value function $\text{pval}_\matcov$ defined for each $\xvec \in \Rset^\dimc$ by:
\begin{equation}\label{eq:RDT-pvalue}
\text{pval}_\matcov(\xvec) = \Marcumc\left(0, \thenorm(\xvec)\right).
\end{equation}
\pastorv{The p-value $\text{pval}_\matcov(\xvec)$ can be seen as a measure of the plausibility of the null hypothesis $\Hcal_0$ given a realization $\xvec \in \Rset^\dimc$ of $\Xvec$ \cite[Sec. 3.3]{Lehmann2005}.}

\vspace{-0.1cm}
\subsection{\pastorv{Weight function for clustering}}
\label{subsec: weight functions}

We now define the weight function that will be used in our clustering algorithm. The expression of this weight function depends on the pvalue function $\text{pval}_\matcov$ defined in~\eqref{eq:RDT-pvalue}.

Henceforth, let $w: [0,\infty) \to [0,\infty)$ be the function defined for any $x \geqslant 0$ by $w(x) = \Marcumc\left(0, \sqrt{x} \right).$
Because this function is continuous, bounded by $1$ and satisfies $\displaystyle \lim\limits_{t \rightarrow \infty} {w(t)} = 0$ \cite{Sun2010}, it satisfies the properties required in Proposition \ref{Prop: fixed points}. 
We therefore choose it as our scalar weight function $w$. Its corresponding $\dimc$-dimensional weight function is therefore defined for any $\xvec \in \Rset^\dimc$ by: 
\begin{equation}
\label{Eq: m-dimensional weight function}
\weight(\xvec) = \Marcumc\left(0, \nu_\matcov(\xvec) \right)
\end{equation}

Proposition \ref{Prop: fixed points} and Claim \ref{claim:model_fp} hold for any matrix $\matcov$ proportional to $\cCorMat$.
From~\eqref{Eq: m-dimensional weight function}, we further observe that the $\dimc$-dimensional weight function $\weight(\xvec)$ depends on the choice of the matrix $\Nc$, which itself depends on the considered Wald test~\eqref{Eq:ddtpb}. 
This is why we now introduce specific Wald tests that will be considered for clustering. These tests will allow us to specify the matrices $\Nc$ that will be used in our algorithms.


\section{Hypothesis tests for clustering}\label{sec:hyp_tests}
In this section, we introduce all the hypothesis tests that will be used in our clustering algorithm. 
The first three tests directly derive from the Wald test introduced in Section~\ref{subsec: Wald} and they will be used mainly for the derivation of the weight functions that are used in our algorithm. 
The fourth considered test will serve to decide whether two estimated centroids $\widehat{\CCtr}_{\ell}$ and $\widehat{\CCtr}_{\ell'}$ actually correspond to the same centroid $ \CCtr_k$. 
Since it is not a Wald test, we completely define it in this section.

\vspace{-0.25cm}
\subsection{Wald's tests for clustering}\label{subsec:specific_wald_tests}
\subsubsection*{Test n°1}
First consider two compressed vectors $\cobs_i \sim \mathcal{N}(\CCtr_{(i)},\cCorMat)$ and $\cobs_j  \sim \mathcal{N}(\CCtr_{(j)},\cCorMat)$, \pastorv{where $\CCtr_{(i)}$ and $\CCtr_{(j)}$ designate the centroids of the clusters to which $\cobs_i$ and $\cobs_j$ belong, respectively.} 
In order to decide whether these two vectors belong to the same cluster, we can test the mean of the vector $ \cobs_i - \cobs_j \thicksim \Ncal(\CCtr_{(i)} - \CCtr_{(j)}, 2\cCorMat)$.
This problem can be solved by the Wald hypothesis test described in Section~\ref{subsec: Wald}, with $\Nc = 2\cCorMat$. 

\subsubsection*{Test n°2}
Now assume that we want to decide whether the compressed vector $\cobs_i$ belongs to cluster $k$ described by centroid $\CCtr_k$. 
This problem can be addressed by testing the mean of the vector $ (\cobs_i - \CCtr_k)~\sim \mathcal{N}(\CCtr_{(i)} - \CCtr_k,\cCorMat)$, which can be solved by the Wald's test of Section~\ref{subsec: Wald} with $\Nc = \cCorMat$. 

\subsubsection*{Test n°3}
Our clustering algorithm will have to test whether $\cobs_i$ belongs to cluster $k$, only knowing an estimate $\widehat{\CCtr}_k$ of $\CCtr_k$. 
According to Claim~\ref{claim:model_fp}, \pastorv{given a positive definite matrix  $\matcov = (1/\coeff) \cCorMat$ with $\coeff \neq 0$,} the estimated centroid  $\widehat{\CCtr}_k$ \pastorv{is} modeled as $\widehat{\CCtr}_k \sim \mathcal{N}(\CCtr_k, ({\NewStd^2}/{\mu^2 N_k}) \cCorMat) $. 
\pastorv{We must thus choose a value of $\mu^2$ to specify the matrix $\Nc$ used in the $\dimc$-dimensional weight function $\weight(\xvec)$.}
In our clustering algorithm described in Section~\ref{sec:algo}, we will \pastorv{actually} consider two $\dimc$-dimensional weight functions $\weight(\xvec)$ \pastorv{specified by two different values of $\mu^2$}. The first weight function will be given by~\eqref{Eq: m-dimensional weight function} with $\pastorv{\coeff =2}$ (Test n°1), and the second one will be given by~\eqref{Eq: m-dimensional weight function} with $\pastorv{\coeff=1}$ (Test n°2).
As a result, we hereafter consider $\coeff \in \{1,2\}$.  

Further, in order to decide whether $\cobs_i$ belongs to cluster $k$, we will assume that $\widehat{\CCtr}_k$ and $\cObs_i$ are independent. In practice, $\widehat{\CCtr}_k$ will be calculated by using a large number of data so that the influence of one $\cObs_i$ can be neglected.
Consequently, in order to make the decision, we will test the mean of the vector $(\cobs_i - \widehat{\CCtr}_k) \sim \mathcal{N}(\CCtr_{(i)} - \CCtr_k, (1 + ({\NewStd^2}/{\mu^2 N_k})) \cCorMat) $. 
This problem can be solved by the Wald hypothesis test described in Section~\ref{subsec: Wald}, with $\Nc = (1 + ({\NewStd^2}/{\mu^2 N_k})) \cCorMat$. Note that if $N_k$ is big, we can approximate $C \approx \cCorMat$, and Test n°3 degenerates into Test n°2.  


\vspace{-0.25cm}
\subsection{Test n°4: hypothesis test for centroid fusion}\label{subsec:test_fusion}
Consider two fixed points $\estCtr_\ell$ and $\estCtr_{\ell'}$ of $\TheFunc$, where $\matcov$ is chosen according to Test n°1 or Test n°2. These fixed points are estimates of two centroids $\CCtr_\ell$ and $\CCtr_{\ell'}$. 
The centroid estimation method used in our algorithm will sometimes result in estimating several times the same centroid. 
This is why our algorithm will also contain a fusion step that will have to decide whether  $\estCtr_\ell$ and $\estCtr_{\ell'}$ are estimates of the same centroid. The fusion step will thus have to decide whether $\CCtr_\ell$ and $\CCtr_{\ell'}$ are different or not, in which latter case $\estCtr_\ell$ and $\estCtr_{\ell'}$ should be merged. 
Merging estimates of two different centroids may result in an artifact significantly far from the two true centroids. 
On the other hand, failing to merge estimates of the same centroid will only result in overestimating the number of centroids. 
This is why we would like to devise an hypothesis test from which the null hypothesis is $\Hcal_0: \CCtr_\ell \neq \CCtr_{\ell'}$ rather than $\CCtr_\ell = \CCtr_{\ell'}$ as in the Wald test. With this choice for the null hypothesis, the alternative hypothesis is $\Hcal_1: \CCtr_\ell = \CCtr_{\ell'}$.

In order to test $\Hcal_0$ against $\Hcal_1$, we proceed similarly as above by considering $\estCtr_\ell-\estCtr_{\ell'}$. In contrast to the three tests discussed in the previous subsections, the random vector $\estCtr_\ell-\estCtr_{\ell'}$ is not necessarily Gaussian. Indeed, $\estCtr_\ell$ and $\estCtr_{\ell'}$ are not independent under either $\Hcal_0$ or $\Hcal_1$. Therefore, the Wald test \cite[Definition III \& Proposition III, p. 450]{Wald1943} does not apply. However, under $\Hcal_0$, by considering that the scalar weight function $w$ tends to put significantly smaller weights on data from clusters other than $\ell$ (resp. $\ell'$) to calculate $\estCtr_\ell$ (resp. $\estCtr_{\ell'}$), we assume the independence of $\estCtr_\ell$ and $\estCtr_{\ell'}$. By further taking the statistical model of Claim~\ref{claim:model_fp} into account, we thus write that, under $\Hcal_0$, $\estCtr_\ell-\estCtr_{\ell'} \thicksim  \Ncal(\CCtr_{\ell} - \CCtr_{\ell'}, \matcov_{\ell,\ell'})$ with $\matcov_{\ell,\ell'} = (1/N_\ell + 1/N_{\ell'})\NewStd^2 \matcov$. We can then proceed as usual in statistical hypothesis testing by exhibiting a test maintaining the false alarm probability of incorrectly rejecting $\Hcal_0$ below a given significance level $\alpha \in (0,1)$. Specifically, the test defined for every $\xvec \in \Rset^\dimc$ by:
\begin{equation}
\label{Eq:Thresholding test from above}
\test'_{\matcov_{\ell,\ell'}}(\xvec) = \left \{
\begin{array}{lll}
1 & \hbox{ if } & \nu_{\matcov_{\ell,\ell'}}( \xvec ) \leqslant  \TheThreshold{1-\level}\\
0 & \hbox{ if } & \nu_{\matcov_{\ell,\ell'}}( \xvec ) > \TheThreshold{1-\level},
\end{array}
\right.
\end{equation}
where $\TheThreshold{1-\alpha}$ is determined according to \eqref{Eq: Wald threshold}, guarantees a false alarm probability less than or equal to $\alpha \in (0,1)$ for testing $\Hcal_0$ against $\Hcal_1$. Indeed, this false alarm probability is:
\begin{align}
\Pbb \mybig [ & \nu_{\matcov_{\ell,\ell'}}(\estCtr_\ell-\estCtr_{\ell'}) \leqslant \TheThreshold{1-\level} \mybig ] \nonumber \\
& = 1 - Q_{\dimc/2} \left ( \| \MMDec_{\ell,\ell'} \left ( \CCtr_\ell - \CCtr_{\ell'} \right ) \|, \TheThreshold{1-\level}, \right ), \nonumber 
\end{align} 
where $\MMDec_{\ell,\ell'}$ is the whitening matrix of $\matcov_{\ell,\ell'}$. Since the generalized Marcum function increases with its first argument \cite{Sun2010}, $\Pbb \mybig [ \nu_{\matcov_{\ell,\ell'}}(\estCtr_\ell-\estCtr_{\ell'}) \leqslant \TheThreshold{1-\level} \mybig ] \leqslant \alpha$.

\section{Centralized Clustering Algorithm}
\label{sec:algo}
This section describes our centralized clustering algorithm~\algo~that applies to compressed data. 
This algorithm derives from the theoretical analysis introduced in the paper. 
In this section, we first present the three main steps of this algorithm (centroid estimation, fusion, and classification). We then describe each of these steps into details.
We also explain how to choose two empirical parameters that are the false alarm probability $\alpha$ and the stopping condition $\epsilon$, and we discuss their influence on the clustering performance. 

\vspace{-0.25cm}
\subsection{Algorithm description}\label{subsec: description of dctrx}


The objective of our clustering algorithm is to divide the set of received compressed vectors $\Zcal = \{\cobs_1,\cdots,\cobs_N\}$ into $K$ clusters, where $K$ is unknown \emph{a priori}. 
The algorithm can be decomposed into three main steps. 
The first step 
consists of estimating compressed centroids $\{ \widetilde{\CCtr}_1,\cdots, \widetilde{\CCtr}_{K'} \}$ from $\Zcal$.
The centroids $\widetilde{\CCtr}_k$ are estimated one after each other by seeking the fixed points of $\TheFunc$ defined in~\eqref{eq:h} (see Section~\ref{subsubsec: centroid estimation}).
Unfortunately, due to initialization issues, this process 
may estimate several times the same centroids. 
This is why the algorithm then applies a fusion step. 
At this step, the algorithm looks for the estimated $\widetilde{\CCtr}_k $ that correspond to the same centroid by applying Test n°4
to every pair $(\widetilde{\CCtr}_i, \widetilde{\CCtr}_j) \in \{ \widetilde{\CCtr}_1,\cdots, \widetilde{\CCtr}_{K'} \}^2$ (see Section~\ref{subsec:fusion}). This 
yields a reduced set $\{  \widehat{\CCtr}_1,\cdots, \widehat{\CCtr}_{K}\}$ of estimated centroids. 
To finish, the algorithm performs a classification step 
associating each compressed vector $\cobs_i$ to the cluster with the closest centroid (see Section~\ref{subsec:classif}).
We now describe into details each of 
these steps.

\vspace{-0.25cm}
\subsection{Centroid estimation}
\label{subsubsec: centroid estimation}
In this section, we introduce the method we use in order to estimate the centroids one after each other. 
Initialize by $\widetilde{\Ectr}=\{ \varnothing\}$ the set of centroids estimated by the algorithm.
Also, initialize by  $\Mk = \{ \varnothing\}$ the set of vectors $\cobs_i$ that are considered as marked, where a marked vector cannot be used anymore to initialize the estimation of a new centroid.

The centroids are estimated one after the other, until $\Mk = \Zcal$. 
When the algorithm has already estimated $k$ centroids, we have $\widetilde{\Ectr} = \{ \widetilde{\CCtr}_1, \cdots, \widetilde{\CCtr}_k   \}$. 
In order to estimate the $k+1$-th centroid, the algorithm picks a measurement vector $\cobs_{\star}$  at random in the set $ \Zcal \setminus \Mk$  and initializes the estimation process with $\widetilde{\CCtr}_{k+1}^{(0)} = \cobs_{\star} $.
In order to estimate $\widetilde{\CCtr}_{k+1}$ as a fixed point of $\TheFunc$~\eqref{eq:h}, the algorithm should recursively compute $\widetilde{\CCtr}_{k+1}^{(\ell+1)} = \TheFunc(\widetilde{\CCtr}_{k+1}^{(\ell)}) $, see~\cite{zoubir12robust}. Here, we consider the following strategy for the matrix $\Nc$ that is used in the recursion. 
In our algorithm, the first iteration is computed as $\widetilde{\CCtr}_{k+1}^{(1)} = h_{2\cCorMat}(\widetilde{\CCtr}_{k+1}^{(0)}) $. This corresponds to $\Nc = 2\cCorMat$ as given by Test n°1 in Section~\ref{subsec:specific_wald_tests}, which comes from the fact that the centroid estimation is initialized with $\cobs_{\star}$. 
From iteration $2$, the recursion is computed as $\widetilde{\CCtr}_{k+1}^{(\ell+1)} = h_{\cCorMat}(\widetilde{\CCtr}_{k+1}^{(\ell)})$, which corresponds to $\Nc = \cCorMat$ as given by Test n°2 in Section~\ref{subsec:specific_wald_tests}. This choice comes from the fact that $\widetilde{\CCtr}_{k+1}^{(1)}$ is already a rough estimate of $\CCtr_{k+1}$. Here, it would be better to consider the value of $\Nc$ given by Test n°3 rather than Test n°2, but $N_{k+1}$ cannot be estimated at this stage of the algorithm.
It is worth mentioning that this strategy (changing the matrix $C$ from iteration $1$ to iteration $2$) led to good clustering performance on all the simulations we considered, with various dimensions $d$ and $\dimc$, number of clusters $k$, matrices $A$, etc.

The recursion stops when \pastorv{$\frac{1}{m} \nu_\matcov \left ( \widetilde{\CCtr}_{k+1}^{(\ell+1)} - \widetilde{\CCtr}_{k+1}^{(\ell)} \right ) \leq \epsilon$}, where $C=\cCorMat$ (Test n°2) 
and $\epsilon$ is the stopping condition. 
The newly estimated centroid is given by $\widetilde{\CCtr}_{k+1} = \widetilde{\CCtr}_{k+1}^{(L)}$, where $L$ represents the final iteration. 
To finish, the set of estimated centroids is updated as $\widetilde{\Ectr} = \widetilde{\Ectr} \cup \{ \widetilde{\CCtr}_{k+1}\}$.

Once the centroid $\widetilde{\CCtr}_{k+1}$ is estimated, the algorithm marks all the vectors that belong to cluster $k+1$. 
For this, the algorithm applies Test n°2 of Section~\ref{subsec:specific_wald_tests} to each $\cobs_i - \widetilde{\CCtr}_{k+1}$, $i \in \{1,\cdots, N\}$.
Here again, we apply Test n°2 instead of Test n°3, because the value $N_{k+1}$ cannot be estimated at this stage of the algorithm. As a result, we assume that $\cobs_i - \widetilde{\CCtr}_{k+1} \sim \mathcal{N}(\CCtr_{(i)} - \widetilde{\CCtr}_{k+1}, \Ac \MCov \Ac^T)$.
All the observations $\cobs_i$ that accept the null hypothesis under this test are grouped into the set $\Mk_{k+1}$. The set of marked vectors is then updated as $\Mk \leftarrow \Mk \cup \{ \cobs_{\star} \} \cup \Mk_{k+1}$.
Note that the measurement vector $ \cobs_{\star}$, which serves for initialization, is also marked in order to avoid initializing again with the same vector.
If $\Mk \neq \Zcal$, the algorithm estimates the next centroid $\widetilde{\CCtr}_{k+2}$. Otherwise, the algorithm moves to the fusion step. 

\vspace{-0.25cm}
\subsection{Fusion}\label{subsec:fusion}
Once $\Mk = \Zcal $ and, say, $K'$ centroids have been estimated, the algorithm applies a so-called fusion step to identify the centroids that may have been estimated several times \pastorv{during the centroid \pastorv{estimation} phase}. 
Indeed, in non-asymptotic situations, the estimated centroids issued from the centroid estimation phase are not guaranteed to be remote from each other and experiments show that the estimation phase tends to over-estimate the true number of centroids. \\
\indent At this step, the algorithm first sets $\widehat{\Ectr} = \widetilde{\Ectr}$. It then applies Test n°4 defined in Section~\ref{subsec:test_fusion} to every pair of estimated centroids $(\widetilde{\CCtr}_i, \widetilde{\CCtr}_j) \in \widetilde{\Ectr}^2, i\neq j$.
Since the cluster sizes $N_i$ and $N_j$ required by Test n°4 are unknown, we replace them by estimates $\widehat{N}_i$ and $\widehat{N}_j$. These estimates are obtained by counting the number of vectors respectively assigned to clusters $i$ and $j$ during the marking operation.
When Test n°4 accepts hypothesis $\mathcal{H}_1$, the algorithm sets $ \widehat{\CCtr}_{\min(i,j)} = \frac{\widetilde{\CCtr}_{i} + \widetilde{\CCtr}_{j}}{2}$ and removes $ \widehat{\CCtr}_{j}$ from $\widehat{\Ectr}$. 
At the end, the number of estimated centroids $K$ is set as the cardinal of the final $\widehat{\Ectr}$ and the elements of $\widehat{\Ectr}$ are re-indexed in order to get $\widehat{\Ectr} = \{\widehat{\CCtr}_1, \cdots \widehat{\CCtr}_K\} $.

\vspace{-0.25cm}
\subsection{Classification}\label{subsec:classif}
Once $K$ centroids $\{ \widehat{\CCtr}_1, \cdots \widehat{\CCtr}_K\}$ have been estimated, the algorithm moves to the classification step. 
Denote by $\Ck$ the set of measurement vectors assigned to cluster $k$.
Each vector $\cobs_i \in \Zcal$ is assigned to the cluster $\mathcal{C}_{k'}$ whose centroid $\widehat{\CCtr}_{k'} \in \widehat{\Ectr}$ is the closest to $\cobs_i$, \emph{i.e.}, $\widehat{\CCtr}_{k'} = \arg\min_{\widehat{\CCtr} \in \widehat{\Ectr}} \pastorv{\nu_\matcov \left ( \cobs_i - \widehat{\CCtr} \right )} $, where $C=\cCorMat$ (Test n°2, assuming that $\widehat{\CCtr}_{k}$ is very close to $\CCtr_{k}$). 
Here, using this condition instead of an hypothesis test forces each measurement vector to be assigned to a cluster. 

\vspace{-0.25cm}
\subsection{Empirical parameters}\label{sec:heuristics}
The described algorithm depends on some parameters $\alpha$ and $\epsilon$. In this section, we describe how to choose these parameters. 

\subsubsection{Parameter $\alpha$}
The false alarm probability $\alpha$ participates to the definition of the weight function $w$ in Section~\ref{sec:w_function}. However, we observed in all our simulations that this parameter does not influence much the clustering performance. More precisely, we observed that any value of $\alpha$ equal or lower than $10^{-2}$ leads to the same clustering performance. The parameter $\alpha$ would be more useful in the case of outliers in the dataset, which is out of the scope of the paper. 

\subsubsection{Parameter $\epsilon$}
The parameter $\epsilon$ defines the stopping criterion in the estimation of the centroids. As for the false alarm probability, $\epsilon$ does not influence much the decoding performance, although it can increase the number of iterations for the estimation when it is too small. In our simulations, we observed that $\epsilon$ can be set to any value between $10^{-2}$ and $10^{-5}$ without affecting the clustering performance. 

At the end, our algorithm~\algo~shows three interesting characteristics compared to other existing clustering algorithms. 	
First, it does not require prior knowledge of $K$, since the centroids are estimated one after the other by looking for all the fixed points of the function $\TheFunc$.
Second, it is not very sensitive to initialization, since the fusion step mitigates the effects of a bad initialization. 
Third, the empirical parameters $\alpha$ and $\epsilon$ do not influence much the clustering performance. 
For these three reasons, the algorithm works in one run and does not need to be repeated several times in order to estimate $K$ and lower the initialization issues (like K-means), or to set up some empirical parameters (like DB-Scan). 
As a result, it appears as a suitable candidate for use in a fully decentralized setup.

\vspace{-0.25cm}
\section{Decentralized Clustering Algorithm}
\label{sec:dalgo}
In this section, we consider a network of $S$ sensors where sensor $s$ observes $N_s$ measurement vectors, and $N=\sum_{s=1}^S N_s$.
We denote by $\Zcal_s = \{\cObs_{s,1}, \cdots \cObs_{s,N_s} \}$ the set of measurement vectors observed by sensor $s$. 
We assume that $\cup_{s=1}^S \Zcal_s = \Zcal $ and that $\Zcal_s \cap \Zcal_{s'} = \{ \phi \}$ for all $s\neq s'$.
We assume that the transmission link between two sensors is perfect, in the sense that no error is introduced during information transmission. Here, for simplicity, we also assume that one sensor can communicate with any other sensor, although the algorithm would apply whatever the communication links between sensors. 
More realistic transmission models will be considered in future works.

In the decentralized algorithm, the operations required by the algorithm are performed by the sensors themselves over the data transmitted by the other sensors. 
The decentralized algorithm is based on the same three steps as the centralized algorithm: centroid estimation, fusion, and classification. 
However, it now alternates between exchange phases at which the sensors exchange some data with each other, and local phases during which each sensor processes its local observations combined with the received data.

We now describe the decentralized version of the algorithm. 
We then evaluate the amount of data exchange needed by our algorithm and compare it with the amount of data exchange required for decentralized K-means. 

\vspace{-0.1cm}
\subsection{Description of the decentralized algorithm~\dalgo}

\subsubsection{Local initializations of the algorithm}
Each sensor $s \in \{1,\cdots, S\}$ first performs a rough clustering on its own data. 
 This rough clustering consists of applying one step of the centralized clustering algorithm as follows. 
The centroids are still estimated one after the \pastorv{other}.
In order to estimate the $k+1$-th centroid, the algorithm picks a measurement vector $\cobs_{\star}$  at random in the set of unmarked vectors for sensor $s$ and produces a new estimated centroid as $\widetilde{\CCtr}_{s,k+1}^{(1)} = \TheFunc(\cobs_{\star})$, with $\Nc = 2\cCorMat$. 
The vector $\widetilde{\CCtr}_{s,k+1}^{(1)}$ constitutes a rough estimate of $\CCtr_{s,k+1}$. 
The algorithm then marks all the vectors that belong to cluster $k+1$ as well as the vector $\cobs_{\star}$ (see centralized algorithm description). When all the vectors in $\Zcal_s$ are marked, the sensor applies a fusion step (see centralized algorithm description), which produces a first set of estimated centroids $\widehat{\Ectr}_s^{(1)} = \{ \widehat{\CCtr}_{s,1}^{(1)}, \cdots \widehat{\CCtr}_{s,K_s}^{(1)} \}$.
Note that the number of estimated centroids $K_s$ can be different from sensor to sensor.

The local algorithm also performs a classification as follows. 
For each $ \cobs_{s,n}$, the algorithm first identifies the estimated centroid $\widehat{\CCtr}_{s,k{\star}}^{(1)}$ that is the closest to $\cobs_{s,n}$, and \pastorv{$k^{\star} = \arg\min_{k \in \{1,\cdots, K_s \} } \nu_\matcov \left ( \cobs_{s,n} - \widehat{\CCtr}_{s,k}^{(1)} \right ) $}, with $C=\cCorMat$. It then applies Test n°2 of Section~\ref{subsec:specific_wald_tests} to $(\cobs_{s,n} - \widehat{\CCtr}_{s,k}^{(1)})$. If the test accepts the null hypothesis, the set $\mathcal{C}_{s,k}$ is updated as $\mathcal{C}_{s,k} = \mathcal{C}_{s,k} \cup \{\cobs_{s,n}\}$, where  $\mathcal{C}_{s,k}$ denotes the set of vectors $\cObs_{s,n}$ that belong to cluster $k$ in sensor $s$.
Due to the hypothesis test, it may occur that some of the measurement vectors are not assigned to any cluster because they are too far from all the estimated centroids. It is very likely that they will be assigned to a cluster after a few exchanges between sensors, since these exchanges will refine the centroids estimates.
From this classification, the algorithm constructs a set $\mathbf{\Nbesk}_{s}^{(1)} = \{\Nbesk_{s,1}^{(1)}, \cdots \Nbesk_{s,K_s}^{(1)} \} $, where $\Nbesk_{s,k}^{(1)}$ denotes the number of measurement vectors $\cObs_{s,n}$ that belong to cluster $k$ in sensor $s$.

To finish, the local algorithm produces two sets $\mathbf{P}_s^{(1)} \! = \! \{ P_{s,1}^{(1)}, \cdots, P_{s,K_s}^{(1)} \}$ and $\mathbf{Q}_s^{(1)} \! = \! \{ Q_{s,1}^{(1)}, \cdots, Q_{s,K_s}^{(1)} \}$. They contain the following partial sums exchanged between the sensors: 
\vspace{-0.25cm}
\begin{align}
P_{s,k}^{(1)} & = \sum_{n = 1}^{N_{s}} \weight(\cObs_{s,n} - \widehat{\CCtr}_{s,k}^{(1)})) \cObs_{s,n} , \vspace{-0.2cm}
\\
Q_{s,k}^{(1)} & = \sum_{n = 1}^{N_{s}} \weight(\cObs_{s,n} - \widehat{\CCtr}_{s,k}^{(1)})).
\vspace{-0.1cm}
\end{align}

\vspace{-0.1cm}
\subsubsection{Exchange phase between sensors}
The exchange phase of the algorithm is realized in $T-1$ time slots. At time slot $t \in \{2,\cdots, T \}$, sensor $s$ receives some data from $J$ other sensors.
The data transmitted from sensor $s'$ to sensor $s$ is composed by $\mathbf{P}_{s'}^{(t-1)}$, $\mathbf{Q}_{s'}^{(t-1)}$, $\Ectr_{s'}^{(t-1)} $, $\mathbf{\Nbesk}_{s'}^{(t-1)}$ . 
Before updating the local parameters of sensor $s$, the algorithm must identify the common centroids between sensors $s$ and $s'$. 
For this, for all pair $(k,k')$, $k \in \{1,\cdots K_s\}$, $k' \in \{1,\cdots K_s'\}$, it applies Test n°4 of Section~\ref{subsec:test_fusion} to $(\widehat{\CCtr}_{s,k}^{(t-1)} - \widehat{\CCtr}_{s',k'}^{(t-1)}) \sim \mathcal{N}(\CCtr_{s,k} - \CCtr_{s',k'},r_{k,k'}^2 \Ac \MCov \Ac^T)$, where
\vspace{-0.1cm}
\begin{equation}
 r_{k,k'}^2 = \NewStd^2 \left( {1}/{\Nbesk_{s,k}^{(t-1)}} + {1}/{\Nbesk_{s',k'}^{(t-1)}} \right).
 \vspace{-0.1cm}
\end{equation}
For every pair $(k,k')$ that accepts hypothesis $\mathcal{H}_1$ of Test n°4, the algorithm updates the partial sums of sensor $s$ as 
\vspace{-0.1cm}
\begin{align}
 P_{s,k}^{(t-1)} & = P_{s,k}^{(t-1)} + P_{s',k'}^{(t-1)} ,  \\
 Q_{s,k}^{(t-1)} &= Q_{s,k}^{(t-1)} + Q_{s',k'}^{(t-1)}  .
\vspace{-0.1cm}
\end{align}
If for a given $k'' \in \{1,\cdots, K_{s'} \}$, there does not exist any $k \in \{1,\cdots, K_{s} \}$ for which the pair $(k,k'')$ accepts hypothesis $\mathcal{H}_1$, sensor $s$ adds a new centroid $\widehat{\CCtr}_{s',k''}^{(t-1)} $ to its own set of centroids. In this case, sensor $s$ updates its own sets as $\mathbf{P}_{s}^{(t-1)} = \mathbf{P}_{s}^{(t-1)} \cup \{ P_{s',k''}^{(t-1)}\}$, $\mathbf{Q}_{s}^{(t-1)} = \mathbf{Q}_{s}^{(t-1)} \cup \{Q_{s',k''}^{(t-1)}\}$, $\Ectr_s^{(t-1)} = \Ectr_s^{(t-1)} \cup \{\widehat{\CCtr}_{s',k''}^{(t-1)} \} $, $\mathbf{\Nbesk}_{s}^{(t-1)} = \mathbf{\Nbesk}_{s}^{(t-1)} \cup \{ \Nbesk_{s',k''}^{(t-1)} \}$. 
This permits to create additional centroids that were not detected in the initial dataset of sensor $s$.

Once the $J$ received sets of data have been processed by sensor $s$, this sensor perfoms 1) an estimation step by estimating new centroids as $ \widehat{\CCtr}_{s,k}^{(t)} = P_{s,k}^{(t-1)}/Q_{s,k}^{(t-1)}$ , 2) a fusion step (see centralized algorithm description), in order to produce a new set of estimated centroids $\widehat{\Ectr}_s^{(t)}$, 3) a classification step in order to compute a new set $\mathbf{V}_{s}^{(t)}$ (see local initialization of the decentralized algorithm), 4) a computation of the updated partial sums $P_{s,k}^{(t)} = \sum_{n = 1}^{N_s} \weight(\cObs_{s,n} - \widehat{\CCtr}_{s,k}^{(t)})) \cObs_{s,n}$ and $Q_{s,k}^{(t)} = \sum_{n = 1}^{N_s} \weight(\cObs_{s,n} - \widehat{\CCtr}_{s,k}^{(t)}))$.

\subsubsection{Final local clustering}
The exchange phase stops after $T-1$ steps. Each sensor $s$ outputs a set of centroids $\widehat{\Ectr}_s^{(T)}$. The final classification at sensor $s$ is realized as the classification in the centralized algorithm and it outputs $K_s$ sets $\mathcal{C}_{s,k}$.

\vspace{-0.25cm}
\subsection{Empirical parameters}
In this decentralized version, the empirical parameters $\alpha$ and $\epsilon$ are chosen as in the centralized algorithm, see Section~\ref{sec:heuristics}. 

\vspace{-0.25cm}
\subsection{Number of exchanged messages}
In this section, we evaluate the number of messages exchanged between sensors by~\dalgo~and compare it to the number of messages required by decentralized K-means~\cite{datta09KDE}.
By number of messages, we mean the number of scalar values exchanged between all the sensors during the whole running of the algorithm. 
We use this criterion instead of the number of operations performed by each sensor, since in most cases, sensor energy consumption is mainly due to information transmission rather than information processing.

The algorithm~\dalgo~consists of $T$ time slots, and at each time slot, each of the $N$ sensors receives $J$ sets  $\left\{\mathbf{P}_{s'}^{(t-1)},\mathbf{Q}_{s'}^{(t-1)},\widehat{\Ectr}_{s'}^{(t-1)} ,\mathbf{V}_{s'}^{(t-1)}\right\}$. In the following, we denote by $\bar{K}_1$ the common cardinality of each of the sets  $\mathbf{P}_{s'}^{(t-1)}$, $\mathbf{Q}_{s'}^{(t-1)}$, $\widehat{\Ectr}_{s'}^{(t-1)}$,  $\mathbf{V}_{s'}^{(t-1)}$ . In the algorithm, the cardinality of these sets is not constant among sensors and time slots, as they are given by the number of clusters estimated by each sensor at each time slot. Therefore, $\bar{K}_1$ is a generic parameter that can be chosen as the average or as the maximum number of clusters.
The sets $\mathbf{P}_{s'}^{(t-1)}$ and $\widehat{\Ectr}_{s'}^{(t-1)} $ are composed by vectors of length $m$ while the sets $\mathbf{Q}_{s'}^{(t-1)}$ and $\mathbf{V}_{s'}^{(t-1)}$ are composed by scalar values. As a result, the total number of exchanged messages  is linear with all the involved parameters and can be evaluated as
\begin{equation}
\nonumber
 \Lambda_1 = 2TJN\bar{K}_1 (m+1) .
\end{equation}
We also note that the compression permits to lower the number of exchanged messages by approximately a factor $m/\dim < 1$ since, without compression, we would have $ \Lambda_1 = 2TJN\bar{K}_1 (d+1) $.

The decentralized K-means algorithm of~\cite{datta09KDE} is also composed by $T$ time slots. At each time slot, each sensor receives $J$ sets of data that are equivalent to $\widehat{\Ectr}_{s'}^{(t-1)}$ and $\mathbf{V}_{s'}^{(t-1)}$. 
In order to deal with the initialization issue, the decentralized K-means can be repeated $R$ times. 
Also, since K-means assumes that $K$ is known, we must try different values of $K$ and consider a penalized criterion in order to both estimate $K$ and perform the clustering. Assume that the algorithm tests values of $K$ from $1$ to $\bar{K}_2$. Therefore, the number of messages exchanged by each sensor can be evaluated as
\begin{equation}
\nonumber
 \Lambda_2 = RTJN \left(\sum_{k=1}^{\bar{K}_2} k \right)  (m+1) = R TJ N \frac{\bar{K}_2(\bar{K}_2+1)}{2} (m+1) .
\end{equation}
This time, we observe that the number of exchanged messages is quadratic with the maximum number of clusters $\bar{K}_2$, although this maximum number is usually small compared to parameters $N$ and $m$.

For comparison between~\dalgo~and decentralized K-means, assume that $\bar{K}_1 = \bar{K}_2$, and that $T$ and $J$ are the same for both algorithms, which is usually verified in practice. With this assumption, we see that $\Lambda_2$ is approximately $\frac{R\bar{K}_2}{4}$ times bigger than $\Lambda_1$, which can make a big difference. For instance, assume, as in our simulations, that $\bar{K}_2=10$, and consider two extreme cases $R=1$ and $R=10$. 
With these two extremes, K-means requires $2.5$ to $25$ more message exchanges than our algorithm, which is significant. 
In our simulation results, the value $R=10$ leads to a good level of performance for K-means, while $R=1$ induces a clustering performance degradation.



\section{Experimental results}\label{sec:experiments}
In this section, we benchmark \algo~and~\dalgo~against standard K-means and DB-Scan.

\vspace{-0.1cm}
\subsection{Centralized algorithm}\label{sec:res_centralized}
This section evaluates the performance of CENTRE-X from Monte Carlo simulations. 
We want to verify that our algorithm can retrieve the correct number of clusters, and we want to assess its performance compared to standard clustering solutions K-means and DB-Scan. 
In all our simulations, we consider $d=100$ and the observation vectors $\obs_n$ that belong to cluster $k$ are generated according to the model $\obs_n \sim \mathcal{N}(\CCtr_k, \sigma^2 \Id)$, where $\sigma^2$ is the noise variance.
Here, we consider a diagonal noise covariance matrix $\MCov = \sigma^2 \Id$ for simplicity. In our simulations, we will consider two models for the centroids $\Ctr_k$: a non-sparse model and a sparse model. Each model will correspond to a different matrix $A$.
For the two models, the parameters of~\algo~are always set to $\alpha=10^{-3}$ for the false alarm probability and $\epsilon=10^{-3}$ for the stopping criterion. As discussed in Section~\ref{sec:heuristics}, these parameters do not influence much the clustering performance.
In our simulations, we consider different pairs of values $(m,\sigma)$ and for every considered pair, we run $1000$ trials with new centroids and new measurement vectors at each trial. 
In order to evaluate the capability of our algorithm to retrieve the correct number of clusters, for each trial, the value of $K$ is selected uniformly at random in the set $\llbracket 1, 10 \rrbracket$.

The clustering performance is evaluated with respect to two criteria. The first criterion is the percentage over the $1000$ trials of cases in which the algorithm retrieved the correct number of clusters $K$. 
The second criterion is the Silhouette~\cite{rousseeuw87CAM}, which measures the quality of the clustering itself.
The performance of our algorithm with respect to these two criteria is compared with three other clustering algorithms. It is first compared with the standard K-means for which $K$ is known and with $10$ replicates in order to lower the initialization issues. Second, we consider the K-means algorithm with $10$ replicates and $K$ is unknown. In this case, we run the K-means algorithm for every $K\in \llbracket 1, 10 \rrbracket$ and we apply an AIC criterion~\cite{pelleg2000x} in order to both retrieve $K$ and perform the clustering. Third, we consider the standard DB-Scan, which does not require the value of $K$. 

\begin{onecol}
\subsubsection{Sparse centroids}
\begin{figure*}[t]
\begin{center}
  \subfloat[~]{ \includegraphics[width=.48\linewidth]{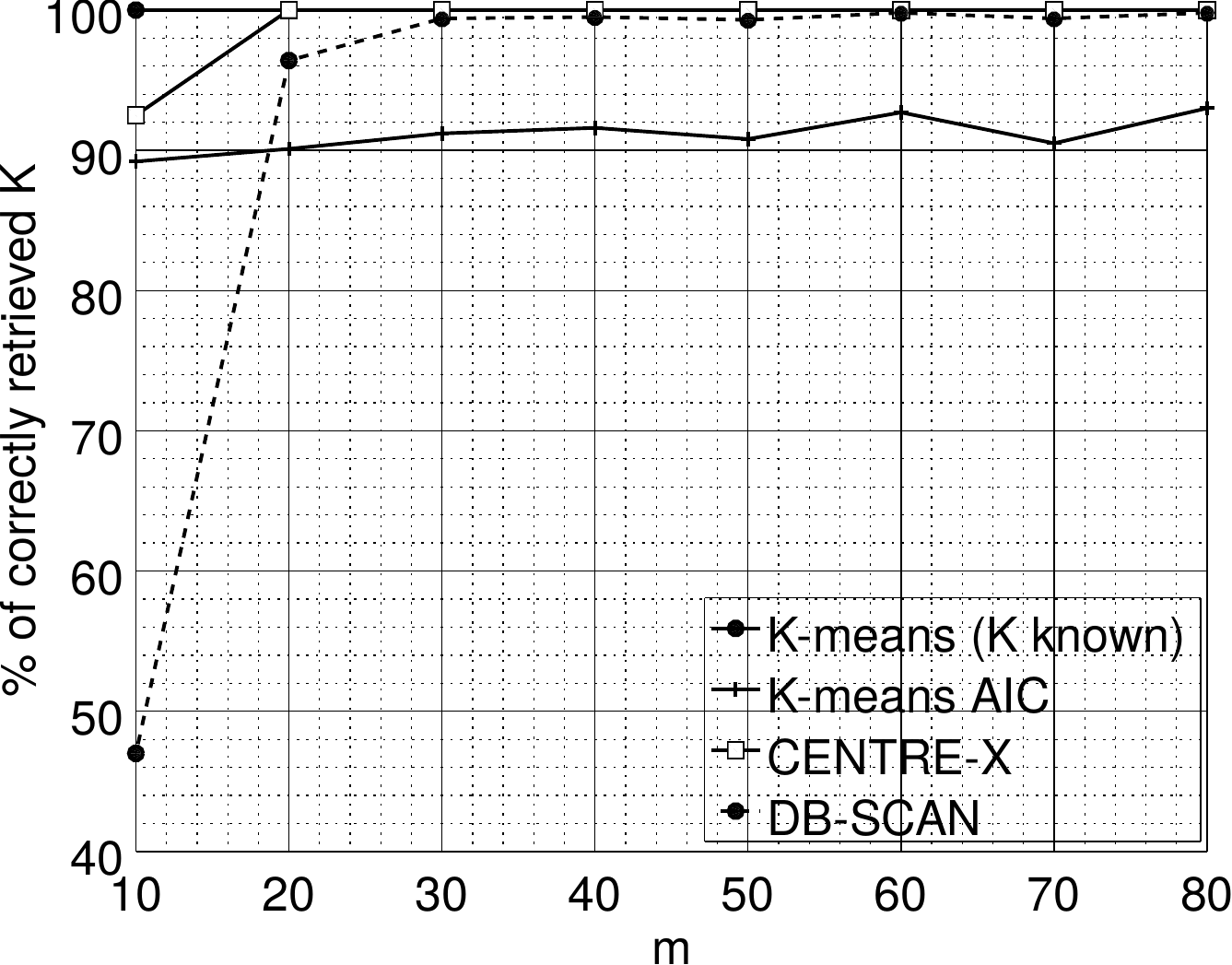}}
  \subfloat[~]{ \includegraphics[width=.48\linewidth]{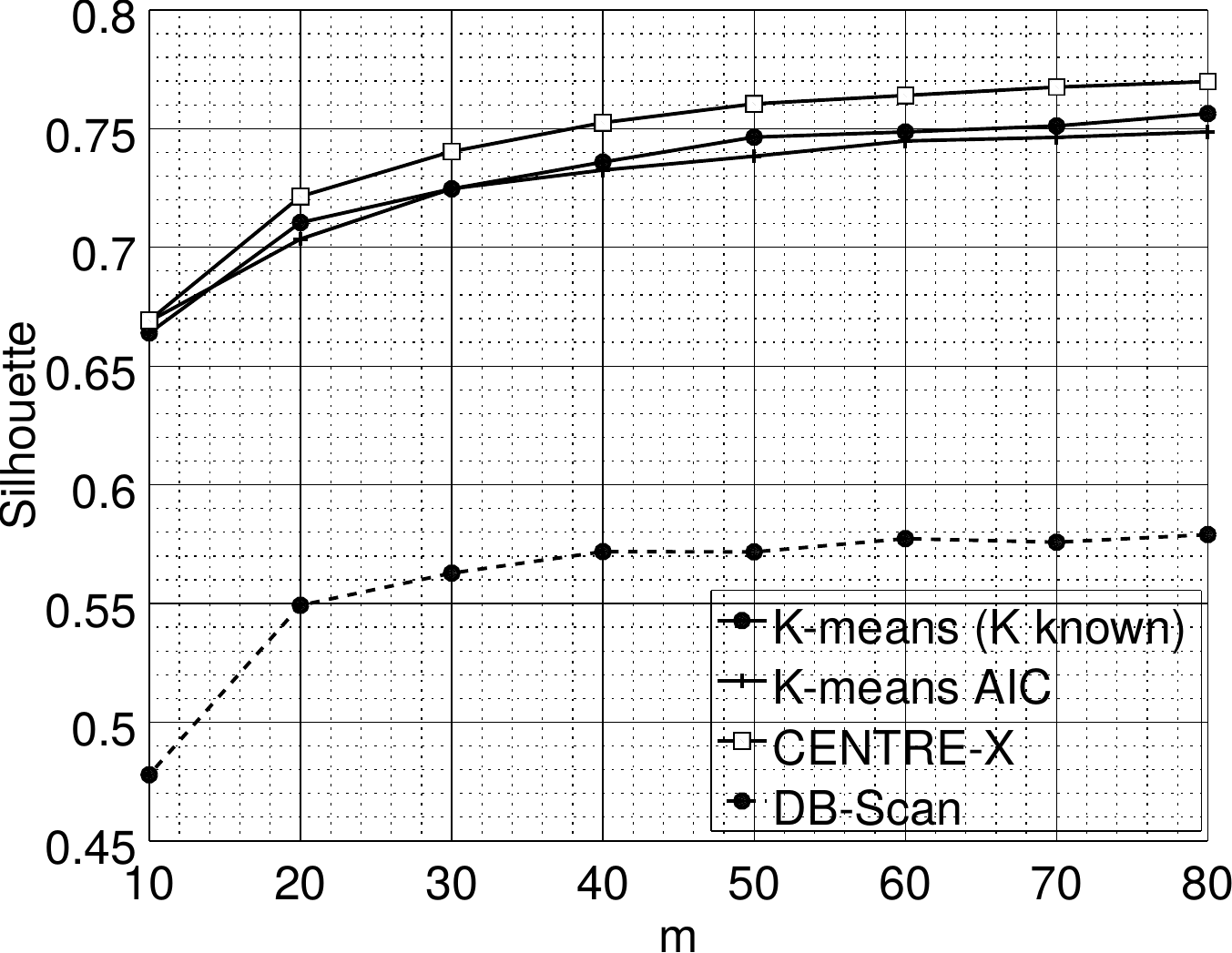}}
\end{center}
\caption{Performance with respect to $m$ of \algo~compared to K-means and DB-Scan in the case of sparse centroids. Sigma is set to $2$ (a) Percentage of correctly retrieved number of  clusters (b) Silhouette}
\label{fig:nonsp_M}
\end{figure*}
\end{onecol}

\begin{twocol}
	\subsubsection{Sparse centroids}
	\begin{figure}[t]
		\begin{center}
			\subfloat[~]{ \includegraphics[width=.5\linewidth]{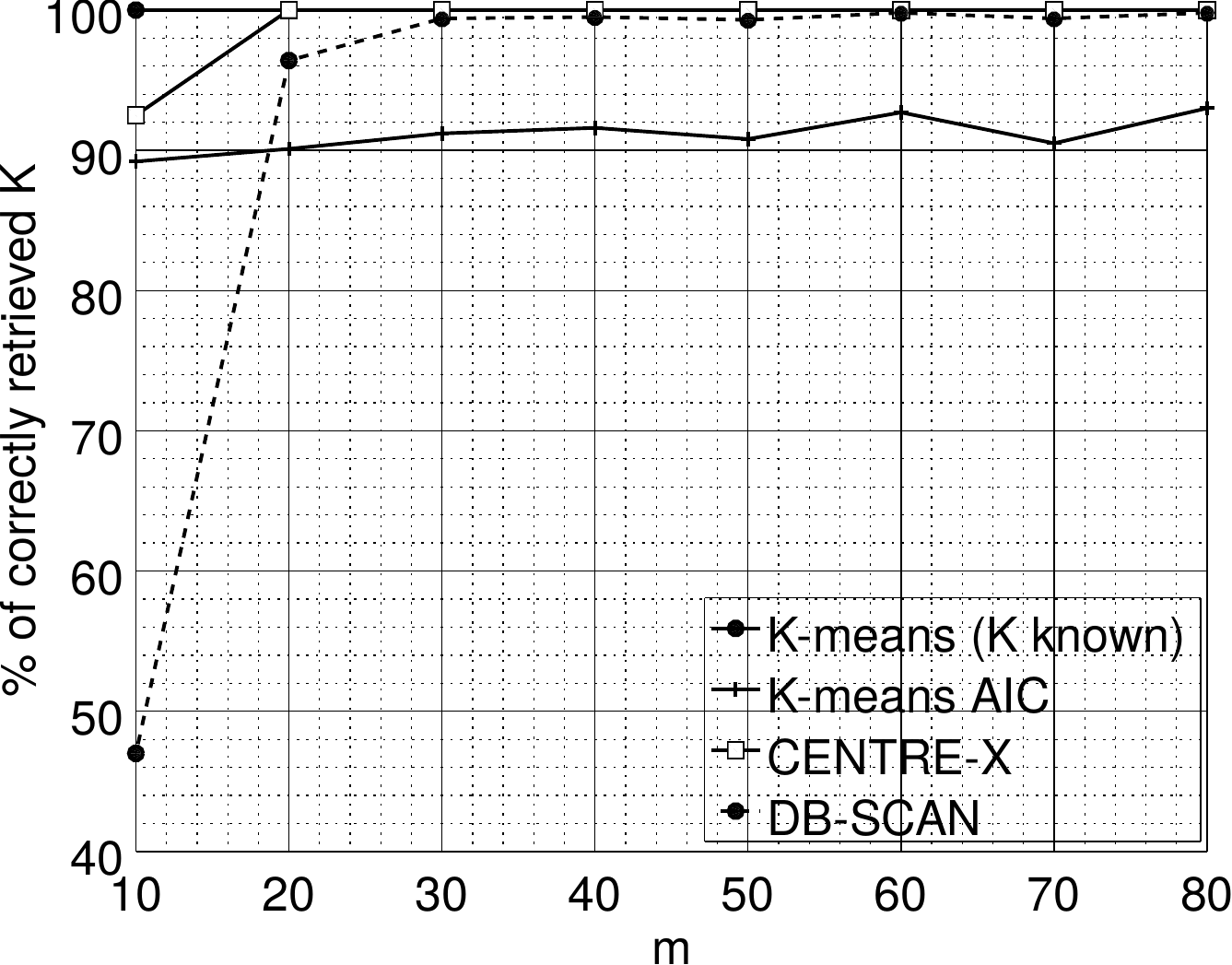}}
			\subfloat[~]{ \includegraphics[width=.5\linewidth]{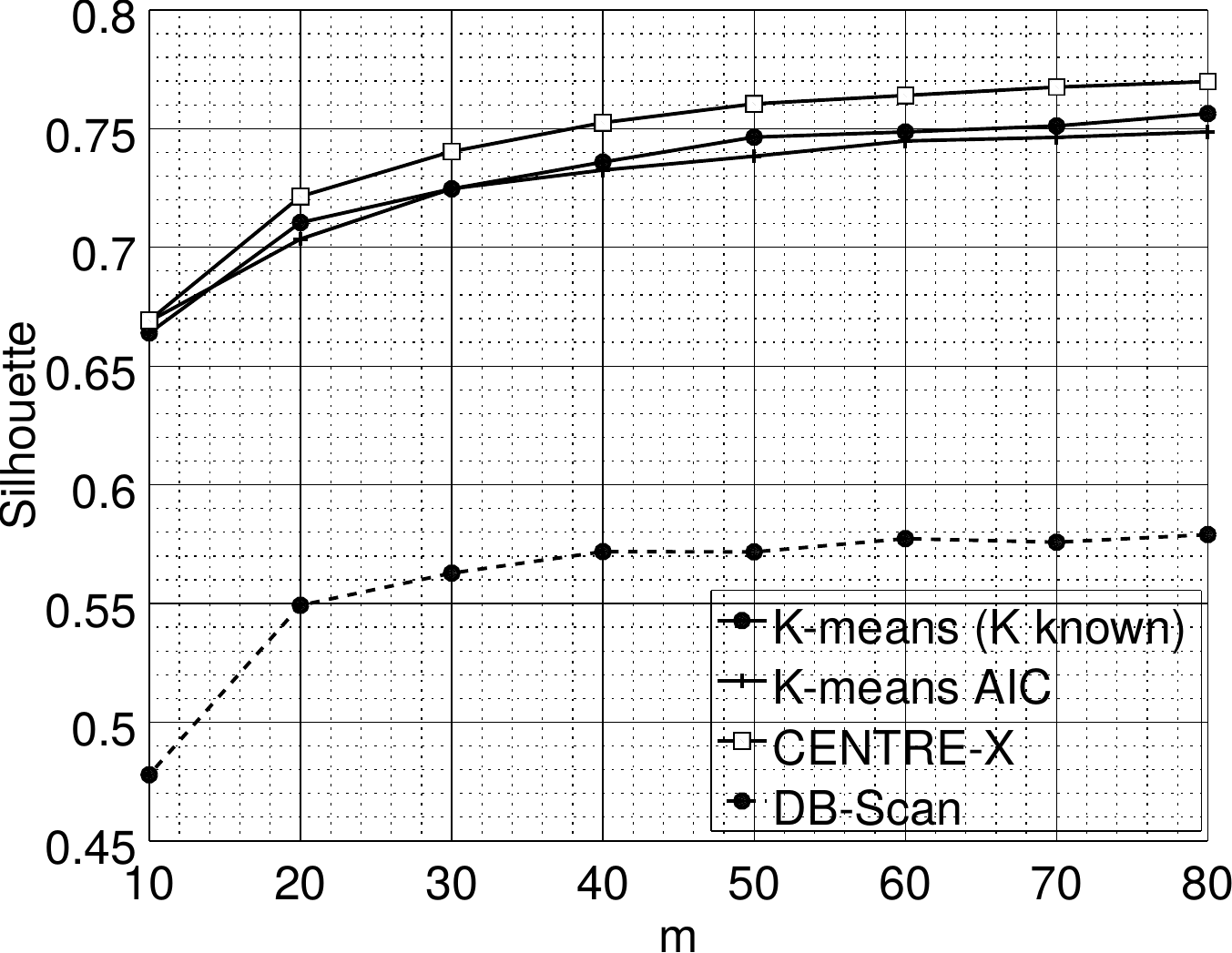}}
		\end{center}
		\caption{Performance with respect to $m$ of \algo~compared to K-means and DB-Scan in the case of sparse centroids. Sigma is set to $2$ (a) Percentage of correctly retrieved number of  clusters (b) Silhouette}
		\label{fig:nonsp_M}
	\end{figure}
\end{twocol}

We first assume a sparse model for the centroids. At each trial, each individual component of each centroid is generated as $\theta_{k,j} \sim \mathcal{N}(0,b^2)$ ($b=2$), with probability $0.2$, and is equal to $0$ otherwise. 
In this case, the matrix $\Ac$ performs random projections with $A_{i,j} \sim \mathcal{N}(0,md)$, $i=1,\cdots,m$ and $j=1\cdots,d$~\cite{zebadua17SP}. 

In Figure~\ref{fig:nonsp_M}, we fix $\sigma=2$ and we represent both the percentage of correctly retrieved clusters and the Silhouette with respect to $m$. 
We first remark that DB-Scan shows an important performance degradation compared to~\algo~and to the two considered versions of K-means. For $m$ larger than $30$, DB-Scan retrieves the correct number of clusters, but the Silhouette value is far from the three other ones. This is probably due to the fact that our data have initial dimensions $\dim=100$ and DB-Scan is known to perform poorly for medium to high dimensions~\cite{hinneburg99Data}.
As a second observation, we see that~\algo~is competitive with respect to the two considered versions of K-means. In particular, it shows the largest Silhouette, even though the three curves are close to each other. It also shows a better ability than K-means AIC to retrieve the correct number of clusters. The K-means AIC algorithm indeed showed difficulties to handle the relatively large set $\llbracket 1, 10 \rrbracket$ of possible values for $K$.

Figure~\ref{fig:nonsp_sigma} considers $m=50$ and represents the clustering performance with respect to $\sigma$. We still observe that DB-Scan shows an important loss in performance. 
We also see that our algorithm has the same performance as K-means for low to intermediate values of $\sigma$. For larger values $\sigma > 4.5$, the percentage of correctly retrieved $K$ starts to decrease for~\algo. However, its Silhouette curve is still close to the two versions of K-means. When $\sigma$ increases, it can occur that two clusters are very close to each other, so that the value of $\sigma$ does not permit to determine whether there are actually two  clusters or one only, see Figure~\ref{fig:ex_clusters} for an example in 2D. This explains why the algorithms does not retrieve the correct value of $K$ but the Silhouette criterion still says that the quality of clustering does not degrade too much.

\begin{onecol}
\begin{figure*}[t]
\begin{center}
  \subfloat[~]{ \includegraphics[width=.48\linewidth]{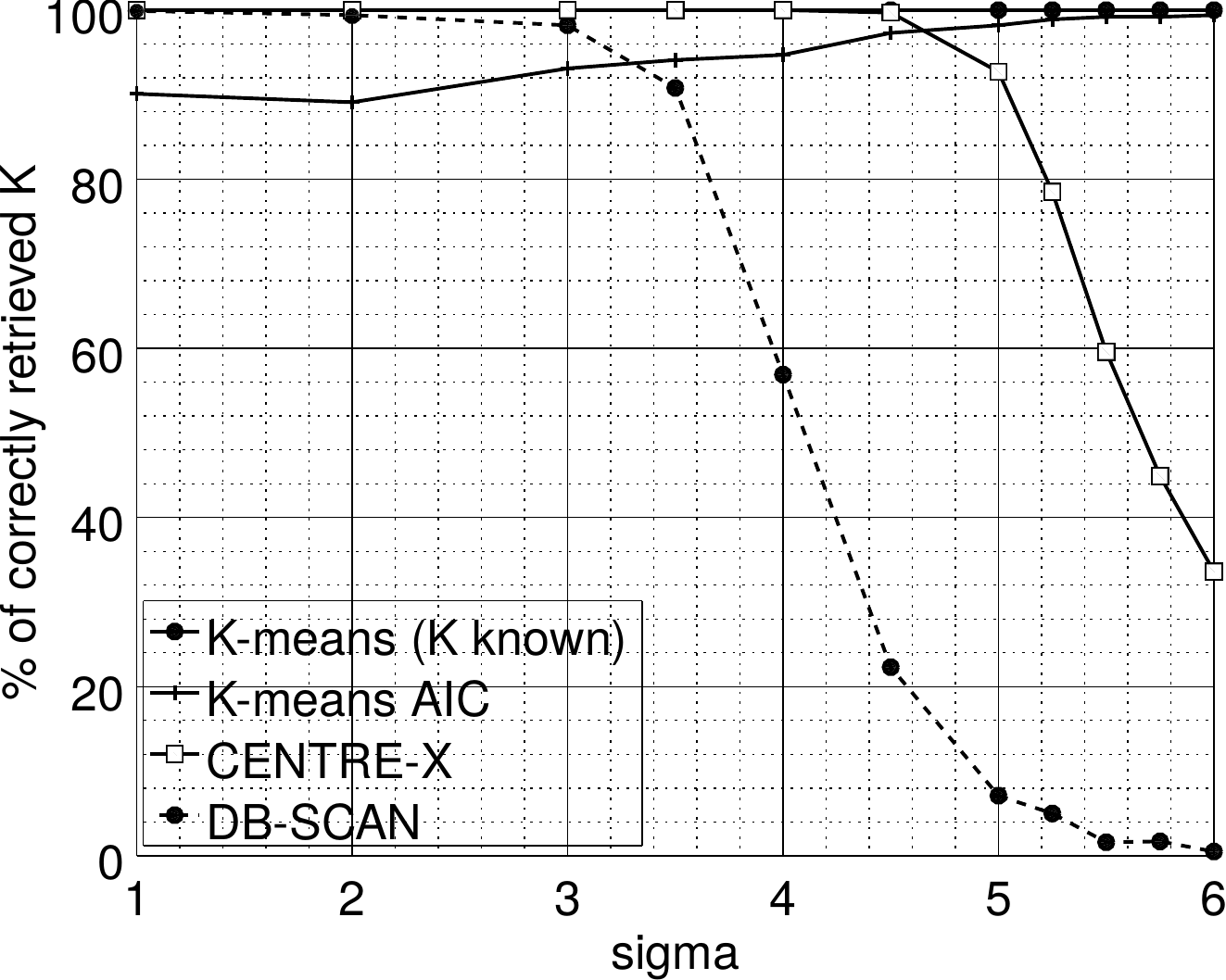}}
  \subfloat[~]{ \includegraphics[width=.48\linewidth]{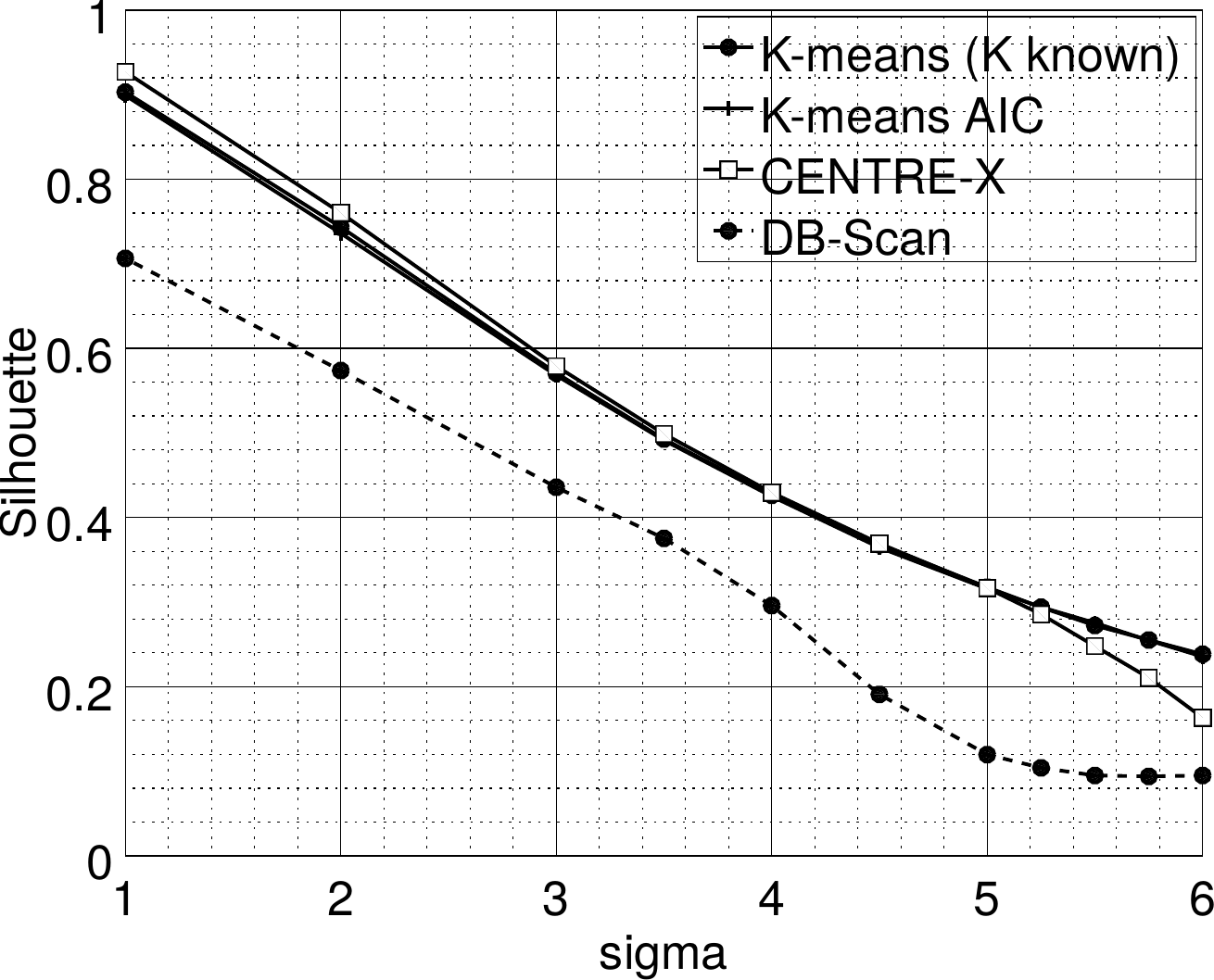}}
\end{center}
\caption{Performance with respect to $\sigma$ of \algo~compared to K-means and DB-Scan in the case of non-sparse signals. The parameter $m$ is set to $50$ (a) Percentage of correctly retrieved number of  clusters (b) Silhouette}
\label{fig:nonsp_sigma}
\end{figure*}
\end{onecol}

\begin{twocol}
	\begin{figure}[t]
		\begin{center}
			\subfloat[~]{ \includegraphics[width=.5\linewidth]{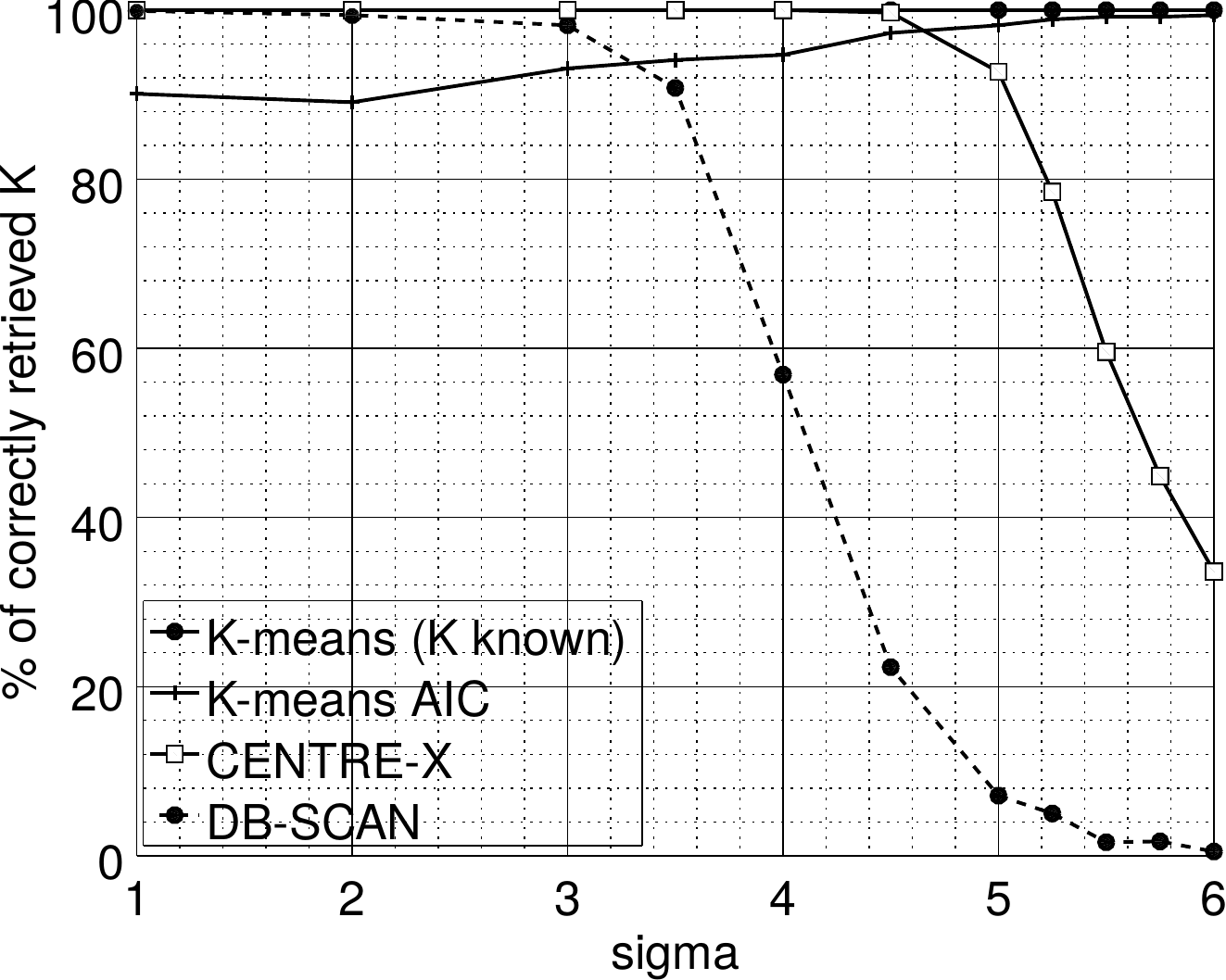}}
			\subfloat[~]{ \includegraphics[width=.5\linewidth]{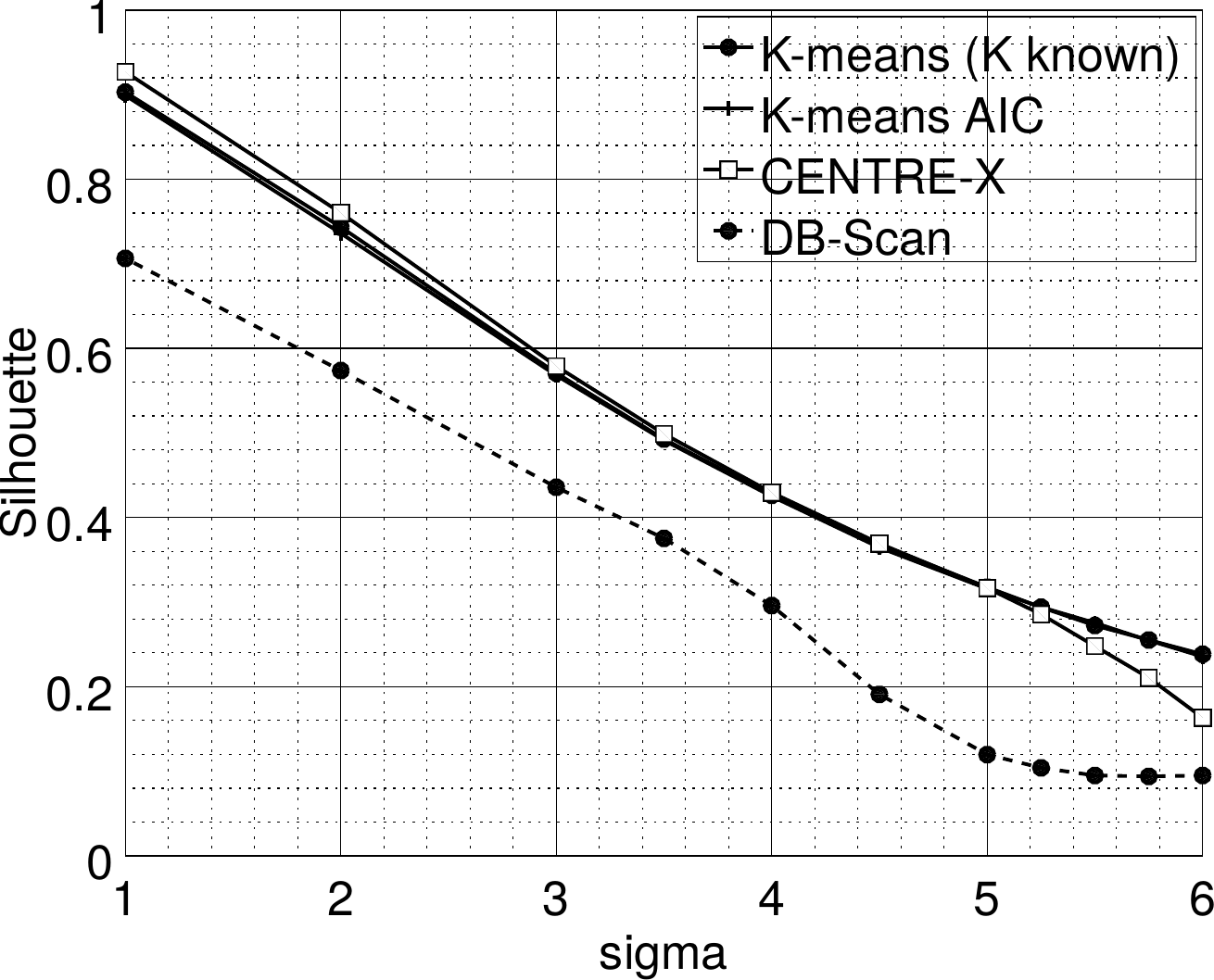}}
		\end{center}
		\caption{Performance with respect to $\sigma$ of \algo~compared to K-means and DB-Scan in the case of non-sparse signals. The parameter $m$ is set to $50$ (a) Percentage of correctly retrieved number of  clusters (b) Silhouette}
		\label{fig:nonsp_sigma}
	\end{figure}
\end{twocol}

\begin{onecol}
	\begin{figure*}[t]
		\centering
		\includegraphics[width=.48\linewidth]{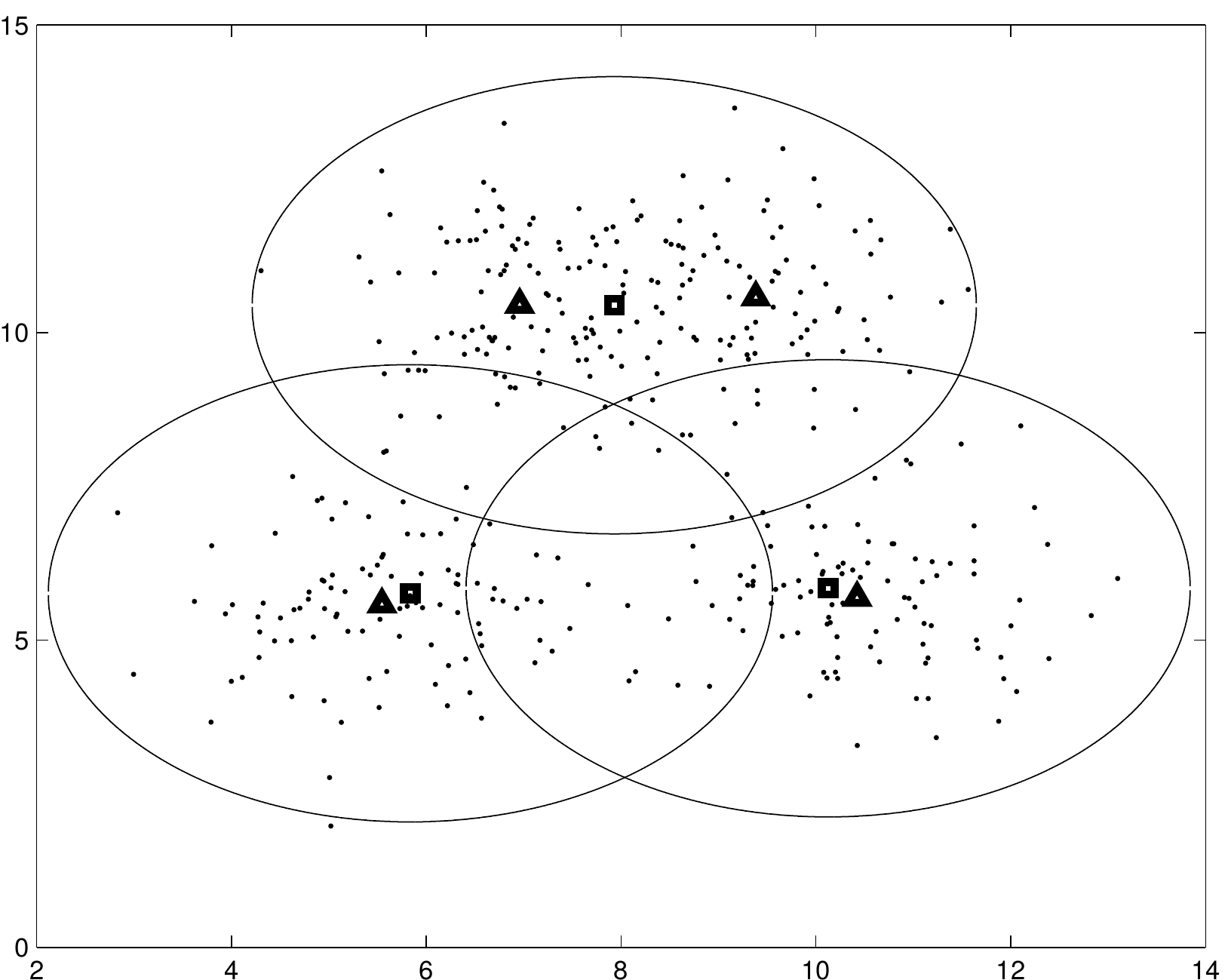}
		\caption{Example of clustering with $d=m=2$, $4$ clusters, and high value of $\sigma$. Triangles give the centroids estimated by K-means initialized with $K=4$ and squares give centroids estimated by our algorithm. The circles correspond to the decision thresholds for our algorithm to decide whether a measurement vector belongs to a cluster.}
		\label{fig:ex_clusters}
	\end{figure*}
\end{onecol}

\begin{twocol}
\begin{figure}[t]
 \centering
 \includegraphics[width=.5\linewidth]{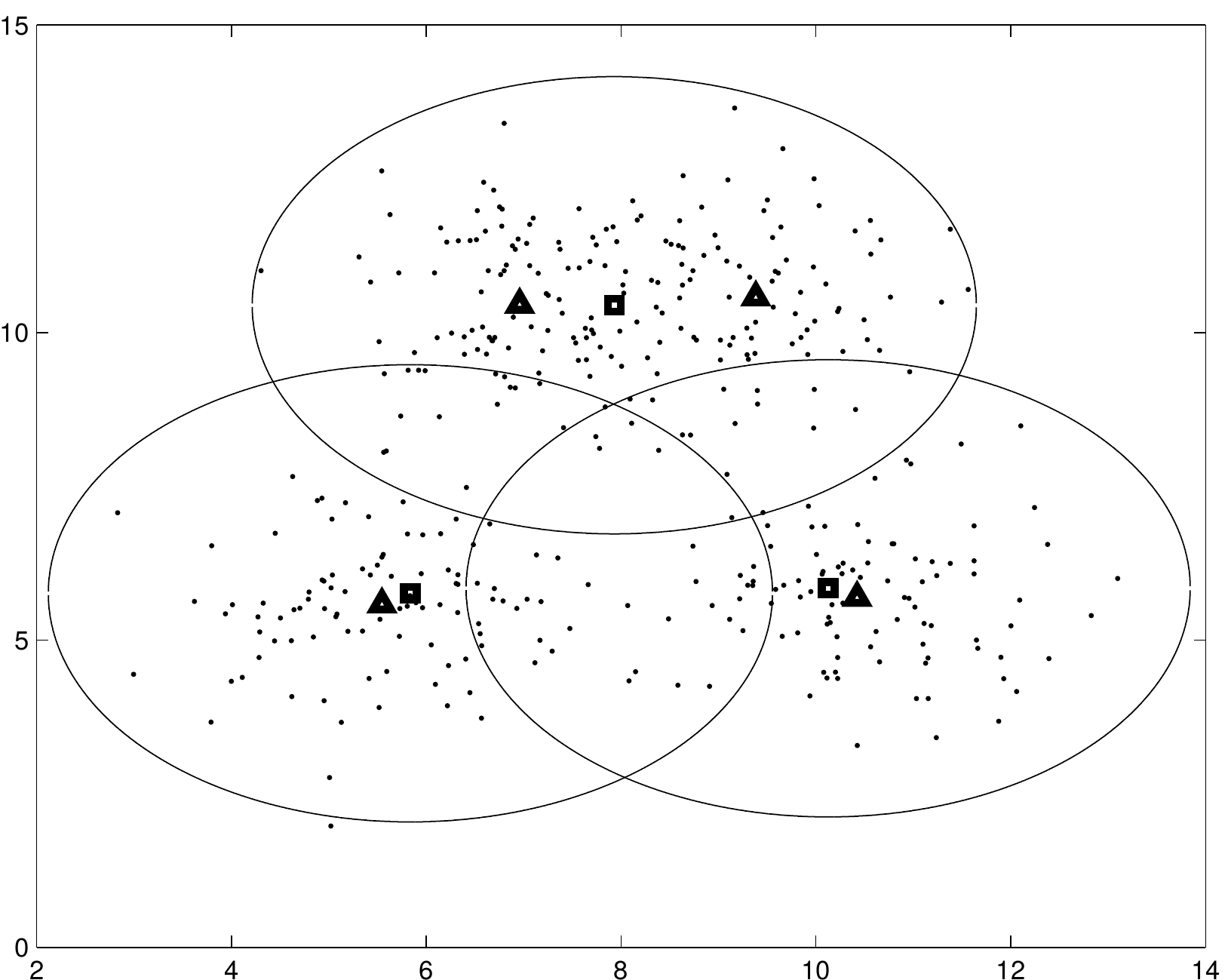}
 \caption{Example of clustering with $d=m=2$, $4$ clusters, and high value of $\sigma$. Triangles give the centroids estimated by K-means initialized with $K=4$ and squares give centroids estimated by our algorithm. The circles correspond to the decision thresholds for our algorithm to decide whether a measurement vector belongs to a cluster.}
 \label{fig:ex_clusters}
\end{figure}
\end{twocol}

\subsubsection{Non-sparse centroids}
In the non-sparse model, we assume that new centroids are generated at each trial as $\Ctr_k \sim \mathcal{N}(0,b^2 \Identity{\dim})$ with $b=2$. 
In this case, the sensing matrix $\Ac$ is constructed so as to randomly select components of $\obs_n$, that is each row of $\Ac$ contains exactly one value $1$, and $0$ elsewhere~\cite{zebadua17SP}.

\begin{onecol}
\begin{figure*}[t]
\begin{center}
  \subfloat[~]{ \includegraphics[width=.48\linewidth]{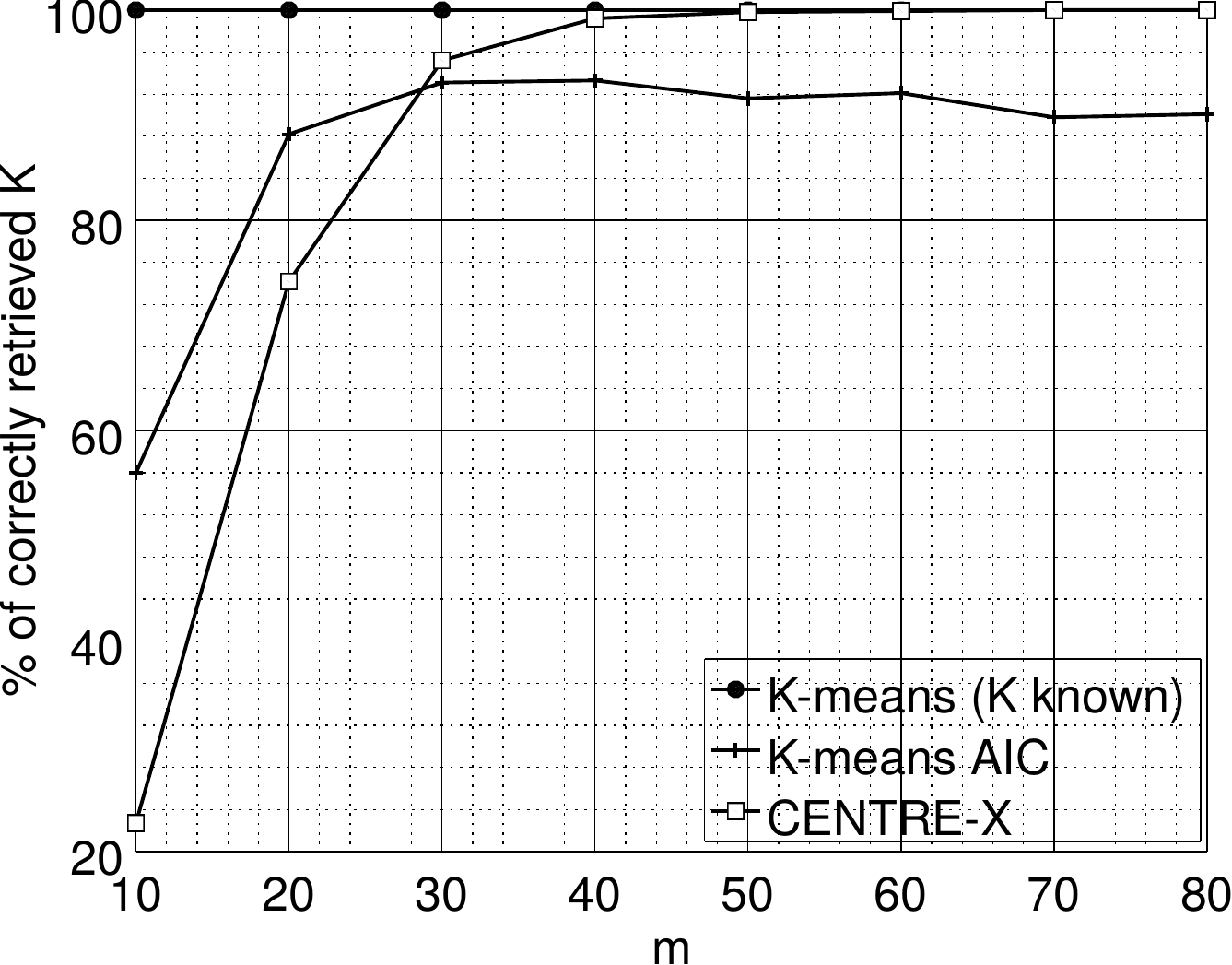}}
  \subfloat[~]{ \includegraphics[width=.48\linewidth]{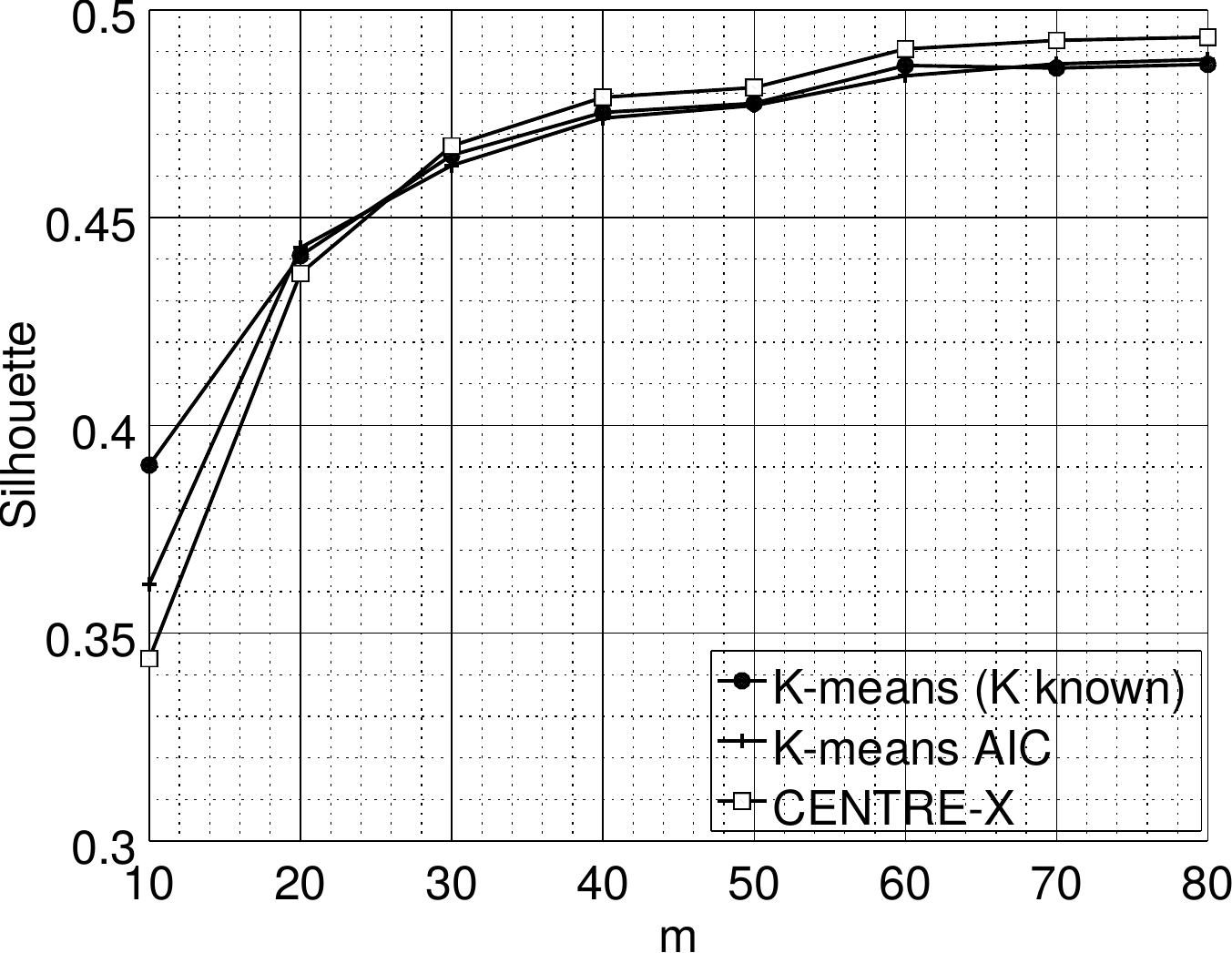}}
\end{center}
\caption{Performance with respect to $m$ of \algo~compared to K-means in the case of sparse centroids. Sigma is set to $2$ (a) Percentage of correctly retrieved number of clusters (b) Silhouette}
\label{fig:sp_M}
\end{figure*}
\end{onecol}

\begin{twocol}
	\begin{figure}[t]
		\begin{center}
			\subfloat[~]{ \includegraphics[width=.5\linewidth]{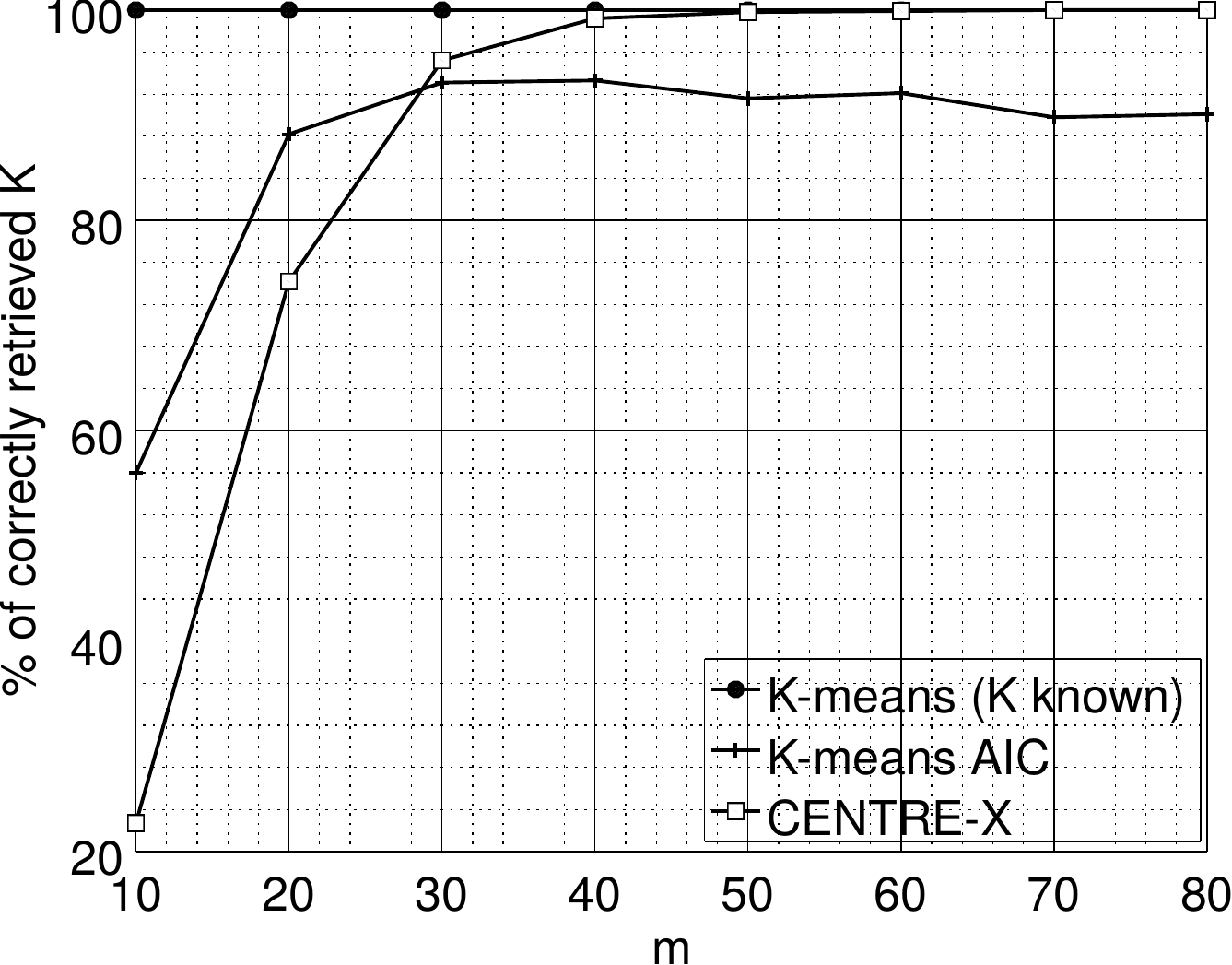}}
			\subfloat[~]{ \includegraphics[width=.5\linewidth]{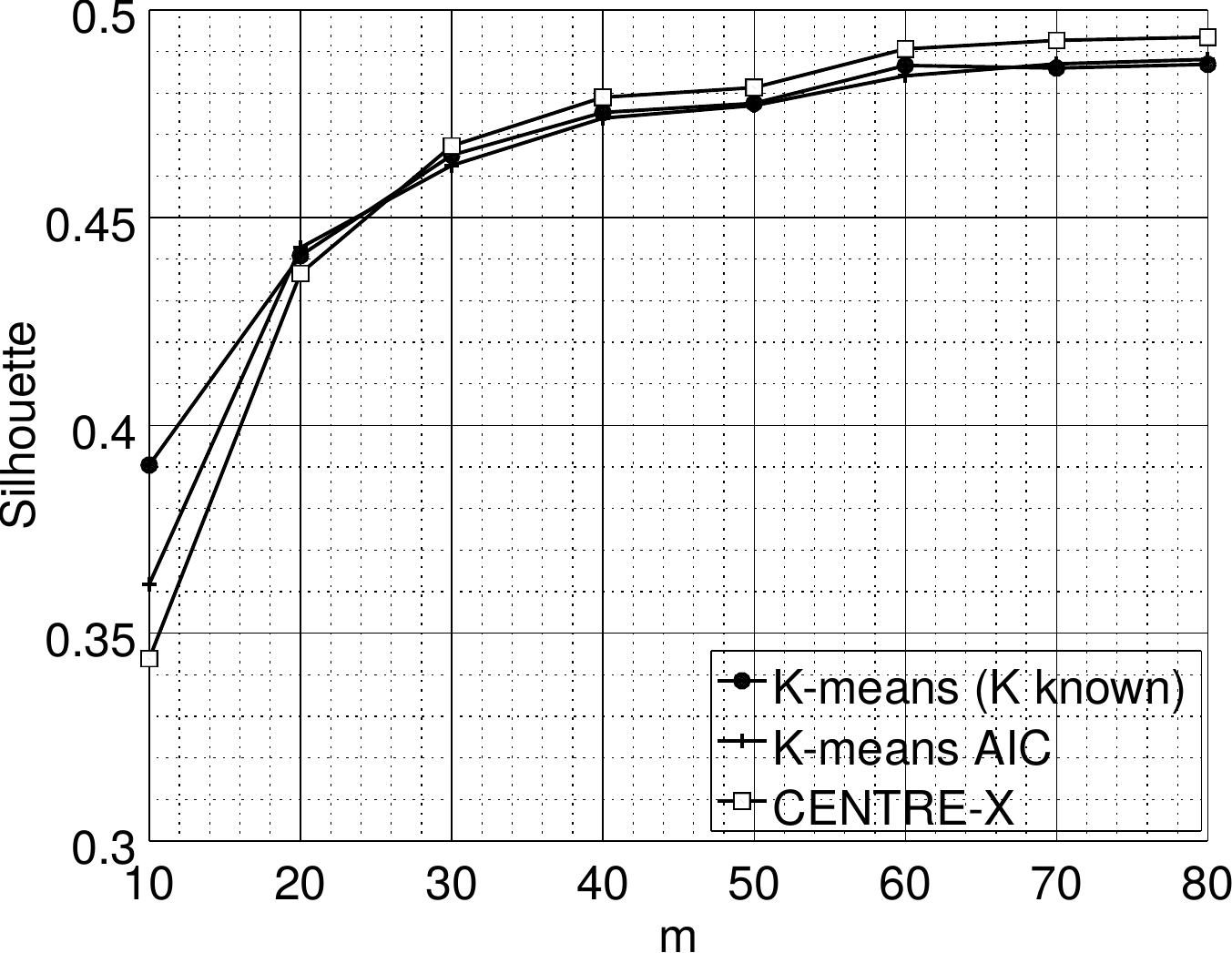}}
		\end{center}
		\caption{Performance with respect to $m$ of \algo~compared to K-means in the case of sparse centroids. Sigma is set to $2$ (a) Percentage of correctly retrieved number of clusters (b) Silhouette}
		\label{fig:sp_M}
	\end{figure}
\end{twocol}

\begin{onecol}
	\begin{figure*}[t]
		\begin{center}
			\subfloat[~]{ \includegraphics[width=.48\linewidth]{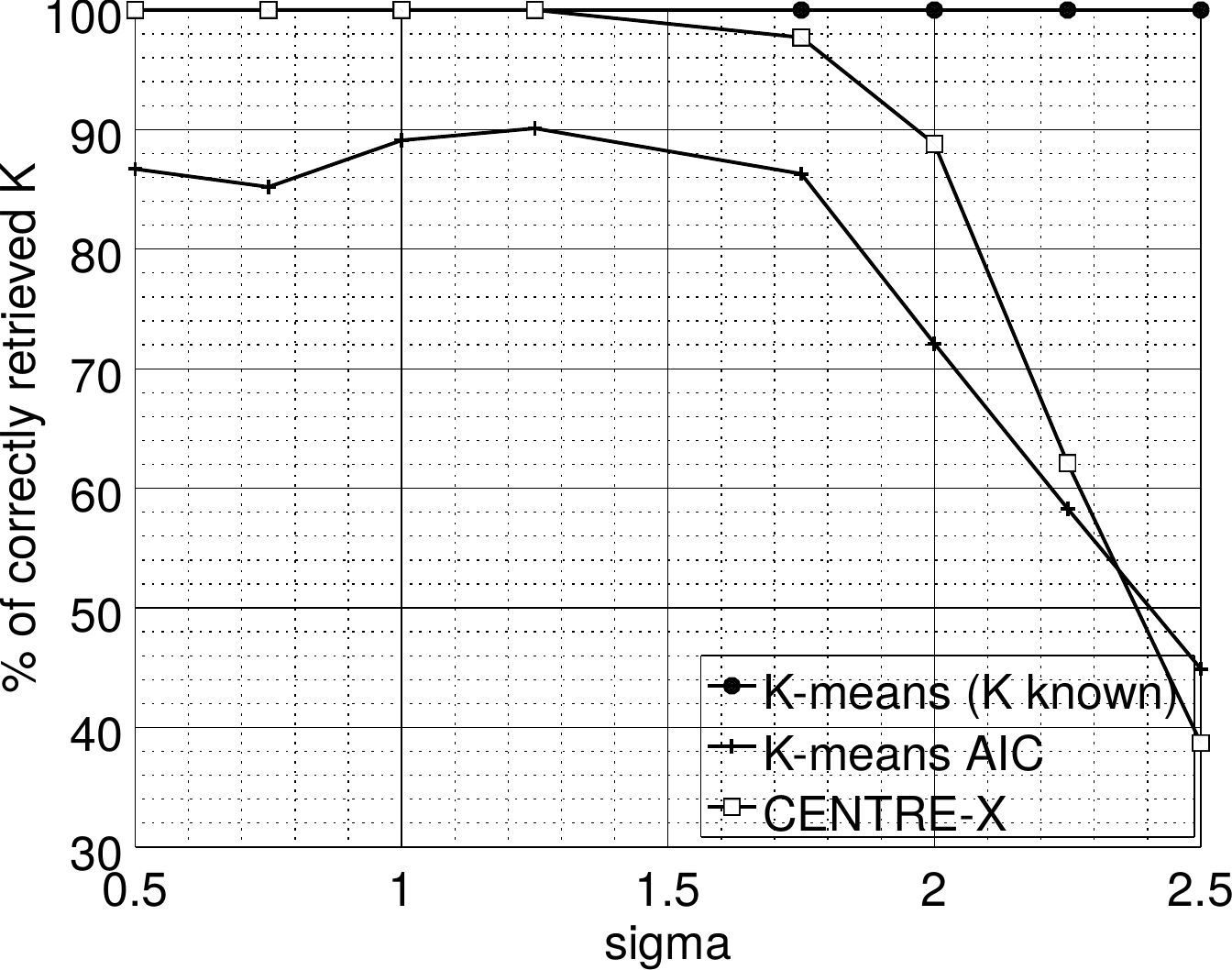}}
			\subfloat[~]{ \includegraphics[width=.48\linewidth]{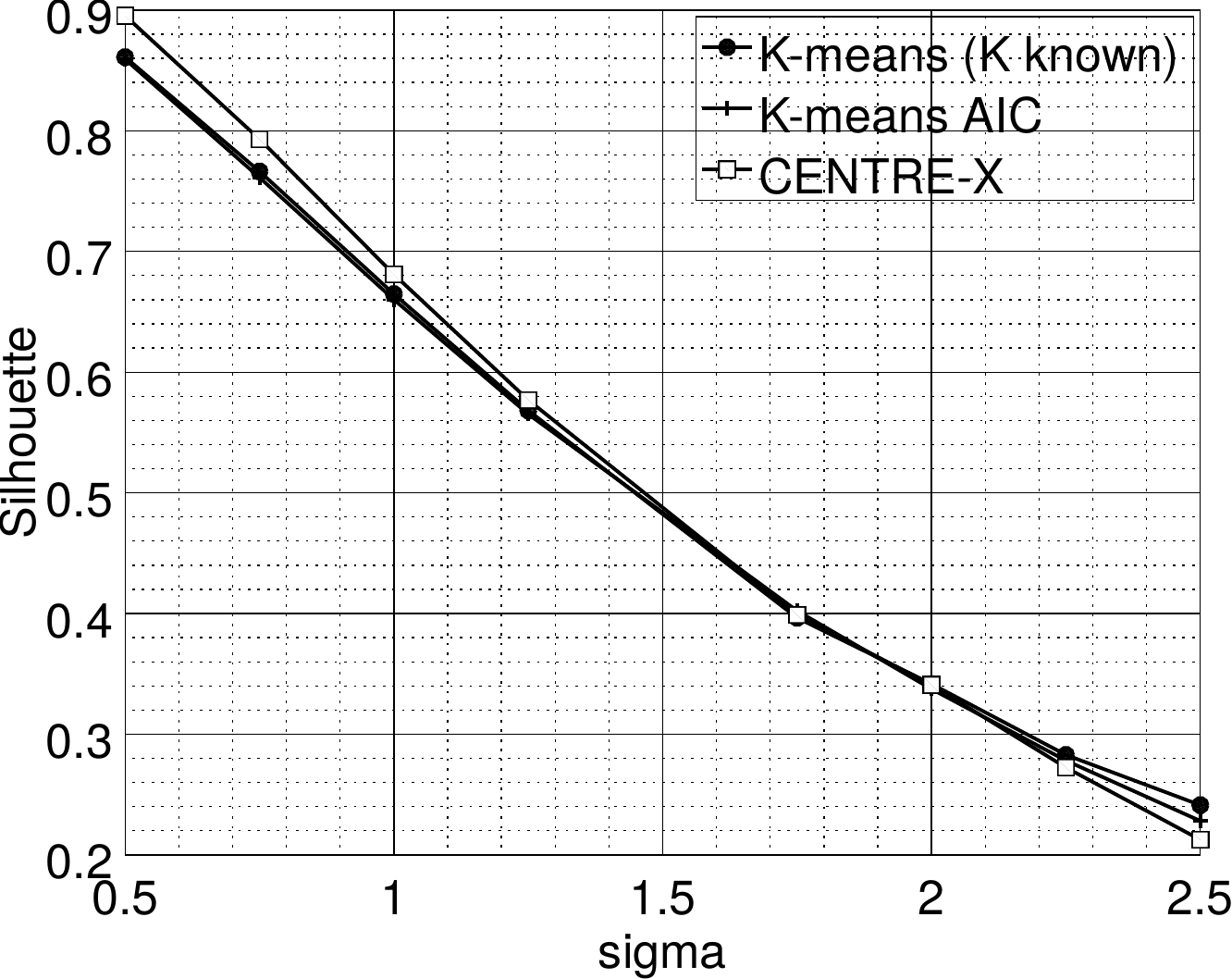}}
		\end{center}
		\caption{Performance with respect to $\sigma$ of \algo~compared to K-means in the case of sparse centroids. $m$ is set to $50$ (a) Percentage of correctly retrieved number of clusters (b) Silhouette}
		\label{fig:sp_sigma}
	\end{figure*}
\end{onecol}

\begin{twocol}
\begin{figure}[t]
\begin{center}
  \subfloat[~]{ \includegraphics[width=.5\linewidth]{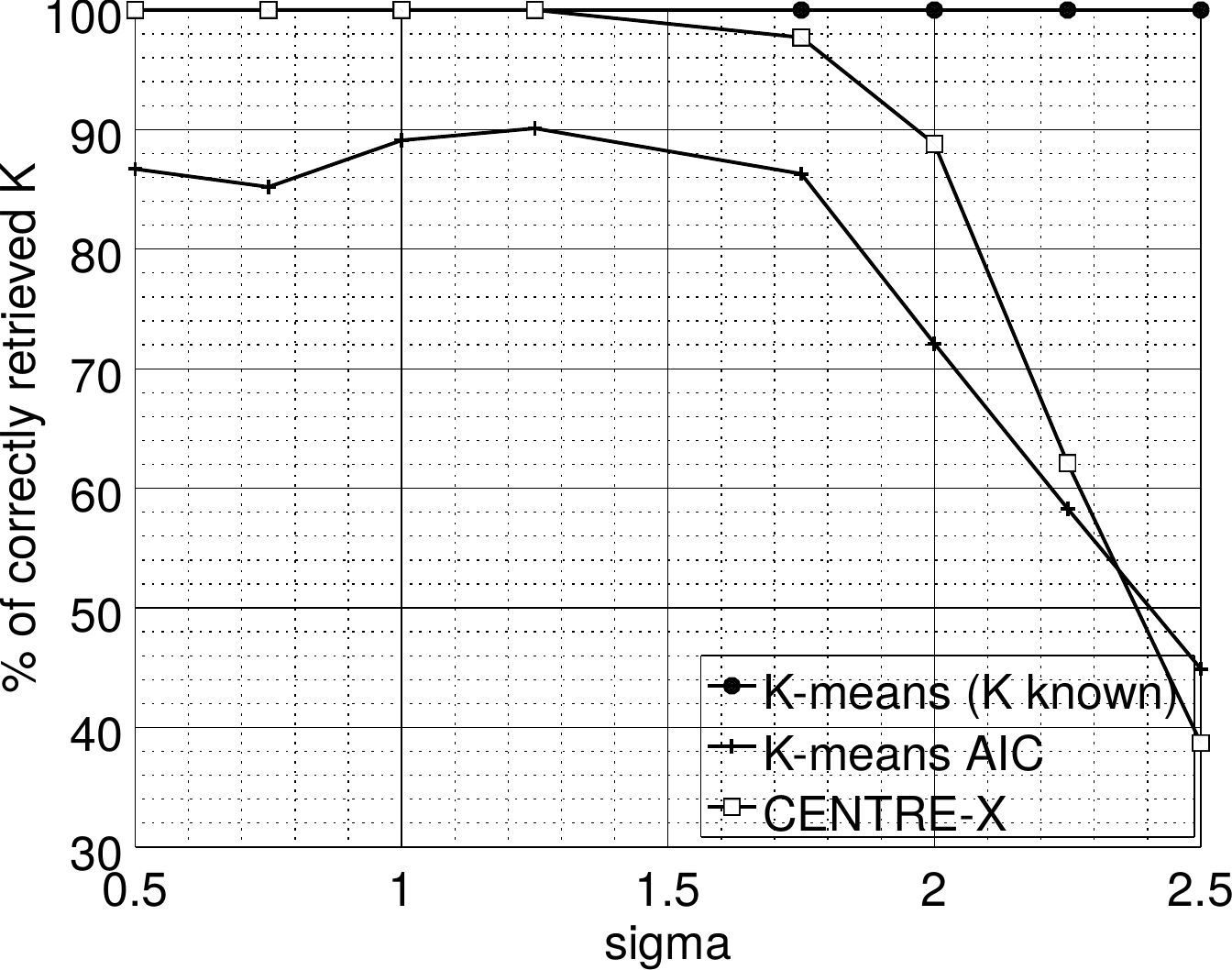}}
  \subfloat[~]{ \includegraphics[width=.5\linewidth]{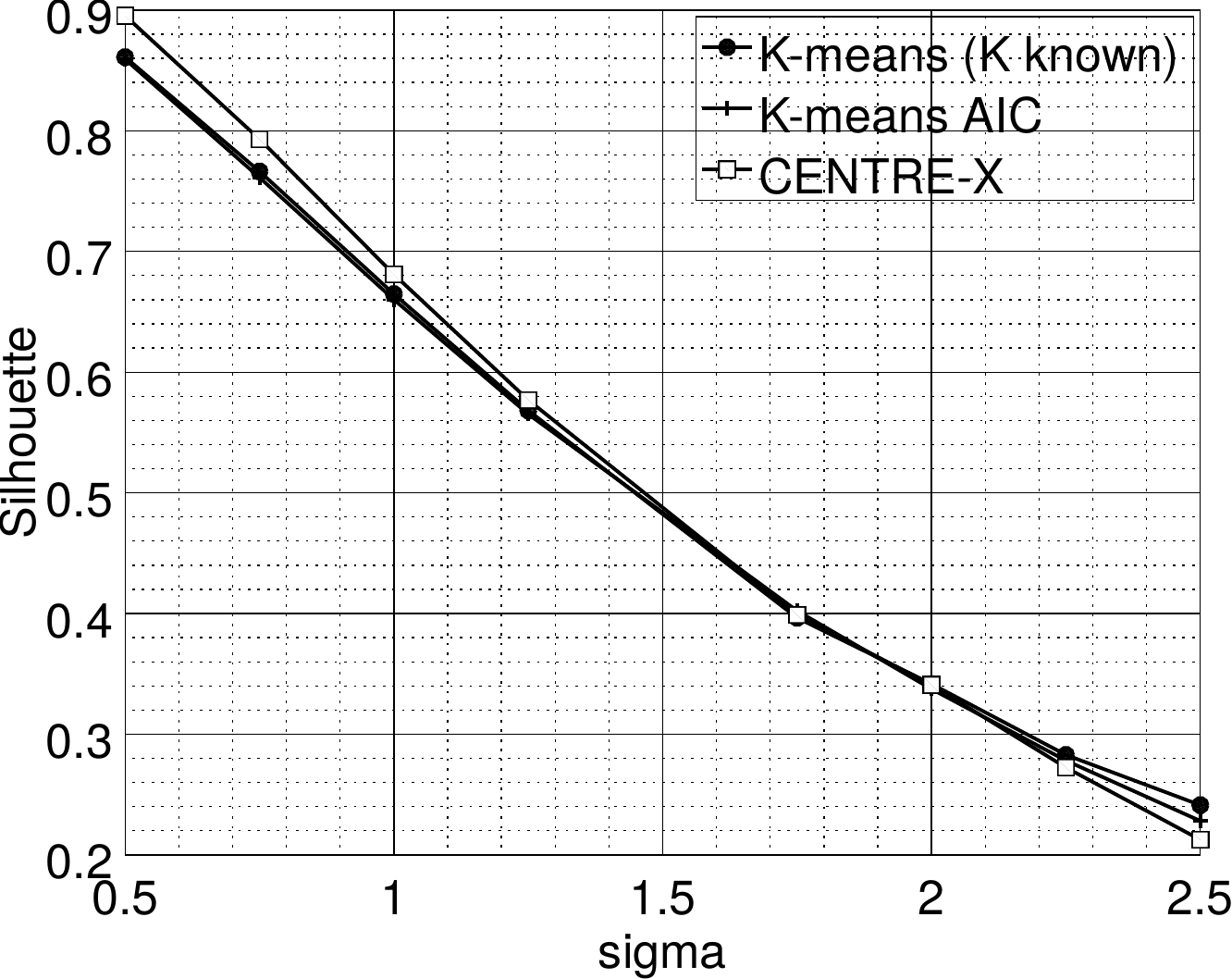}}
\end{center}
\caption{Performance with respect to $\sigma$ of \algo~compared to K-means in the case of sparse centroids. $m$ is set to $50$ (a) Percentage of correctly retrieved number of clusters (b) Silhouette}
\label{fig:sp_sigma}
\end{figure}
\end{twocol}

With this model, Figure~\ref{fig:sp_M} considers $\sigma=2$ and shows the two considered criteria with respect to $m$. In this case, we see that our algorithm shows a degradation compared to K-means AIC for small values of $m$, but that it performs better for larger values of $m$. This is mitigated by the fact that the three Silhouette curves are still very close to each other. 
Figure~\ref{fig:sp_sigma} considers $m=50$ and shows the performance with respect to $\sigma$. In this case, our algorithm shows a performance degradation compared to K-means with $K$ known for the largest values of $\sigma$, but it outperforms K-means AIC for almost all the considered values of $\sigma$. The Silhouette curves are also very close to each other. 
As for the non-sparse model, the performance degradation in terms of correctly retrieved $K$ of~\algo~for small $m$ and large $\sigma$ can be explained by proximity of clusters.

There results show that~\algo~is competitive with K-means when $K$ is unknown and incurs only a small performance loss compared to K-means with K known. In contrast to K-means, CENTRE-X requires no prior knowledge of the number of clusters and suffers from no initialization issues (no need for replicates). 
This makes our algorithm a good candidate for fully decentralized clustering.

\vspace{-0.25cm}
\subsection{Decentralized algorithm}

We now evaluate the performance of~\dalgo~and compare it with the performance of the fully decentralized K-means algorithm~\cite{datta09KDE}.
We consider a network with $S=20$ sensors and a dataset of size $N=1000$. We assume that the compressed measurement vectors $\cObs_n$ are equally distributed between sensors, which means that each sensor observes $50$ vectors $\cObs_n$.
The compressed vectors $\cObs_n$ are generated according to the non-sparse model described in Section~\ref{sec:res_centralized}, with the same parameter $b=2$ and the same construction of sensing matrices $A$. As for the centralized algorithm, we also set $\alpha=10^{-3}$ and $\epsilon=10^{-3}$.
Regarding the parameters that are specific to the decentralized algorithms, we set $T=10$ time slots and $J=2$ received sets of data per sensor per time slot.
The parameters $T$ and $J$ are the sames for both~\dalgo~and decentralized K-means.

Figure~\ref{fig:dX_nonsp_M} considers $\sigma=2$ and represents the percentage of correctly retrieved clusters as well as the Silhouettes with respect to $m$.
We first observe that decentralized K-means without replicates shows a performance degradation compared to decentralized K-means with $10$ replicates at all the considered values $m$. 
We then see that~\dalgo~shows a performance degradation when $m<30$ but has the same performance as decentralized K-means with $10$ replicates when $m>30$.

Figure~\ref{fig:dX_nonsp_sigma} considers $m=50$ and represents the percentage of correctly retrieved clusters as well as the Silhouettes with respect to $\sigma$.
We get the same conclusions as for Figure~\ref{fig:dX_nonsp_M}: decentralized K-means without replicates always shows lower performance than K-means with $10$ replicates;~\dalgo~ shows a performance degradation for $\sigma>2.5$ and the same performance as decentralized K-means with $10$ replicates for $\sigma<2.5$.
Recall that decentralized K-means with $10$ replicates requires approximately $25$ times more message exchanges than~\dalgo, and that decentralized K-means without replicates needs approximately $2.5$ times more message exchanges than~\dalgo.
This shows that~\dalgo~offers a clustering solution that is competitive with respect to decentralized K-means.

\begin{onecol}
\begin{figure*}[t]
\begin{center}
  \subfloat[~]{ \includegraphics[width=.48\linewidth]{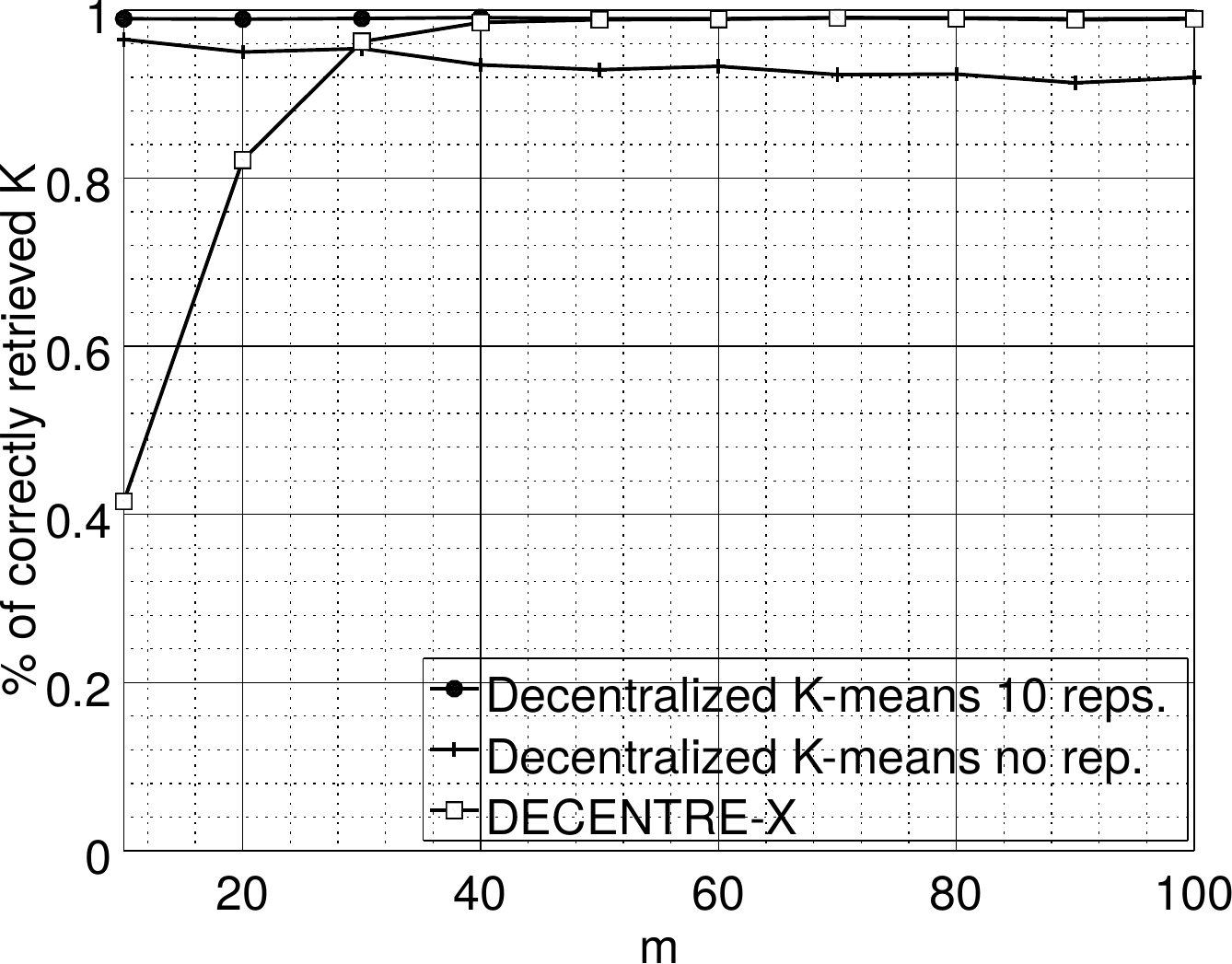}}
  \subfloat[~]{ \includegraphics[width=.48\linewidth]{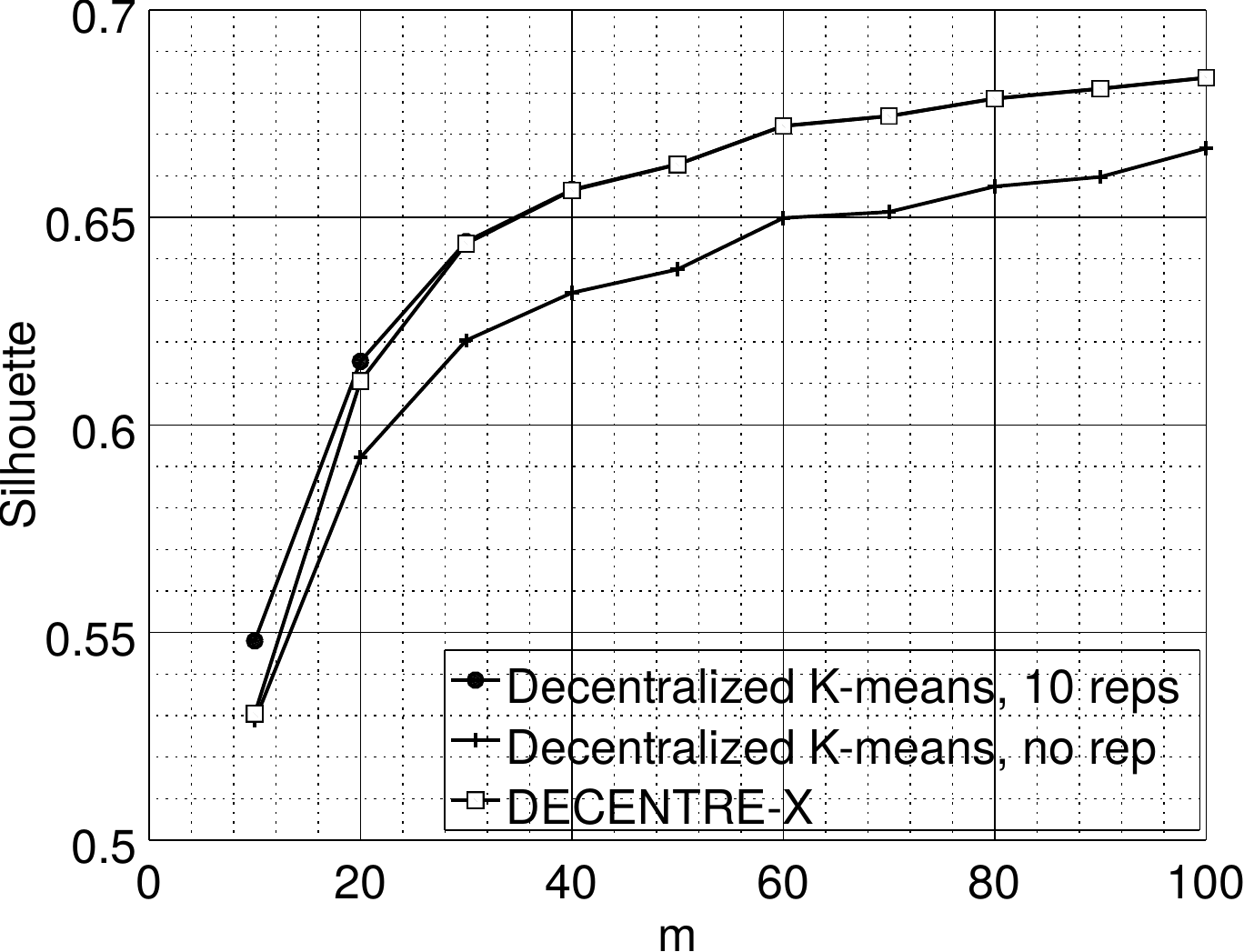}}
\end{center}
\caption{Performance with respect to $m$ of \dalgo~compared to K-means with $R=10$ and $R=1$ for non-sparse centroids. Sigma is set to $2$. (a) Percentage of correctly retrieved number of clusters (b) Silhouette}
\label{fig:dX_nonsp_M}
\end{figure*}
\end{onecol}

\begin{twocol}
	\begin{figure}[t]
		\begin{center}
			\subfloat[~]{ \includegraphics[width=.5\linewidth]{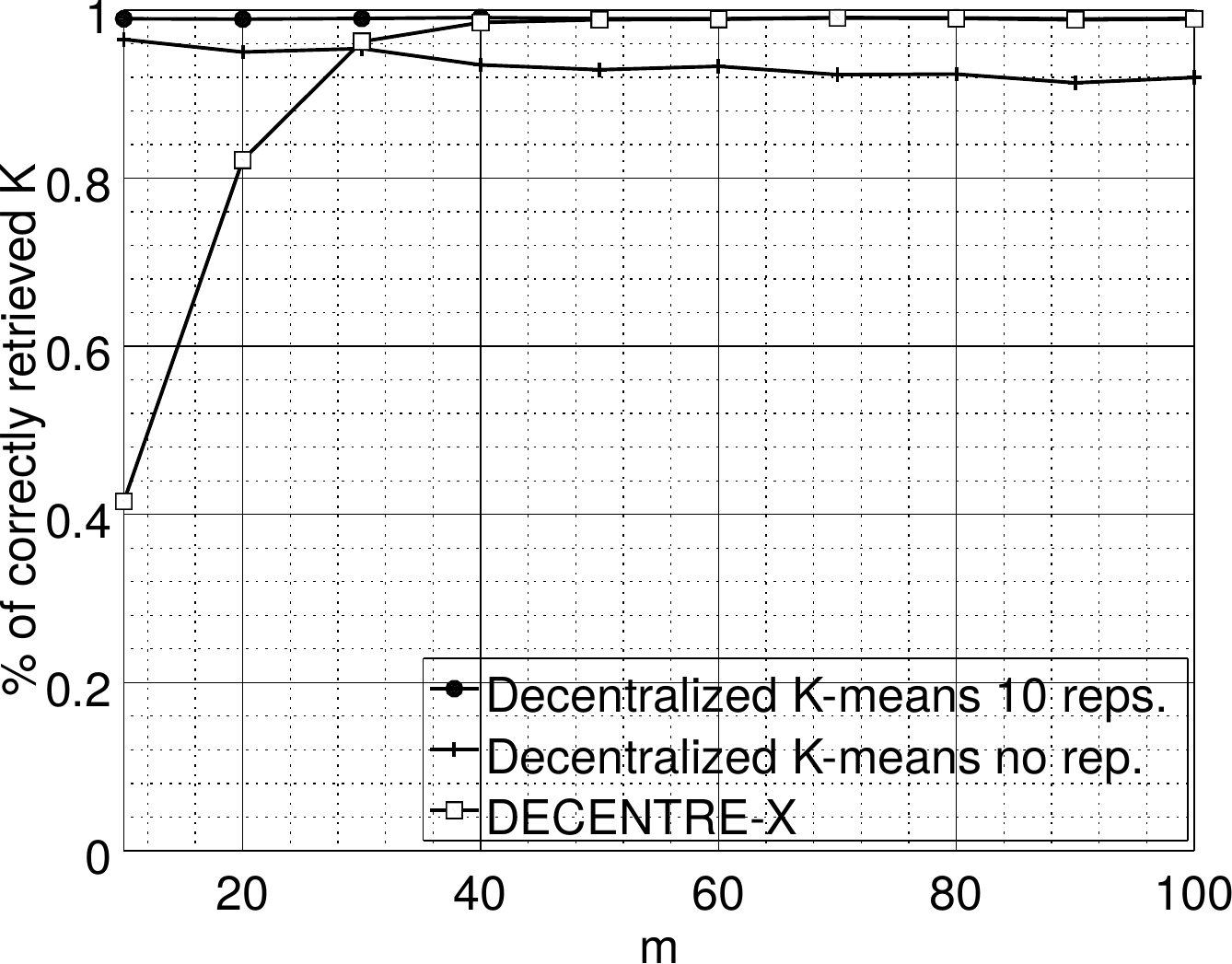}}
			\subfloat[~]{ \includegraphics[width=.5\linewidth]{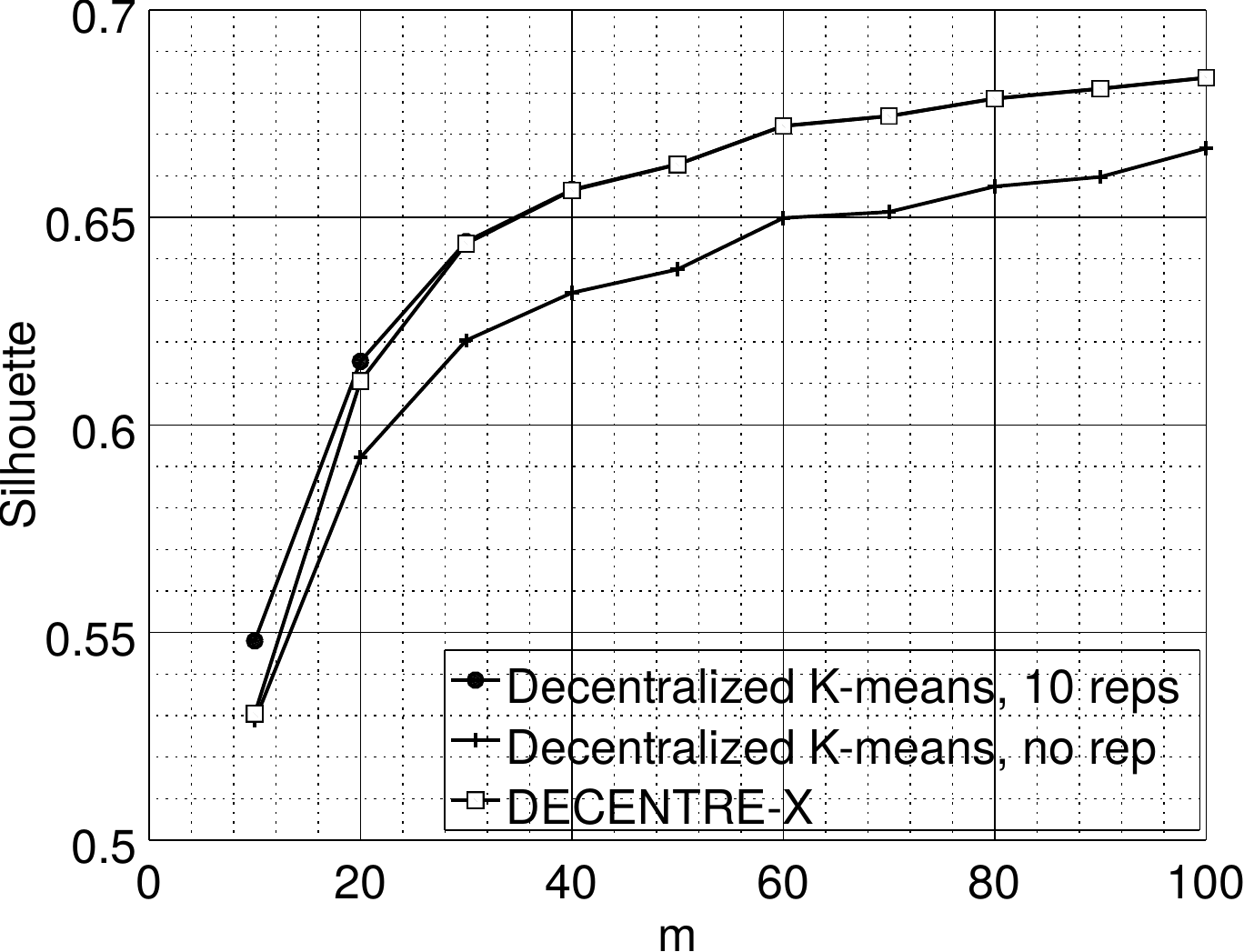}}
		\end{center}
		\caption{Performance with respect to $m$ of \dalgo~compared to K-means with $R=10$ and $R=1$ for non-sparse centroids. Sigma is set to $2$. (a) Percentage of correctly retrieved number of clusters (b) Silhouette}
		\label{fig:dX_nonsp_M}
	\end{figure}
\end{twocol}

\begin{onecol}
	\begin{figure*}[t]
		\begin{center}
			\subfloat[~]{ \includegraphics[width=.5\linewidth]{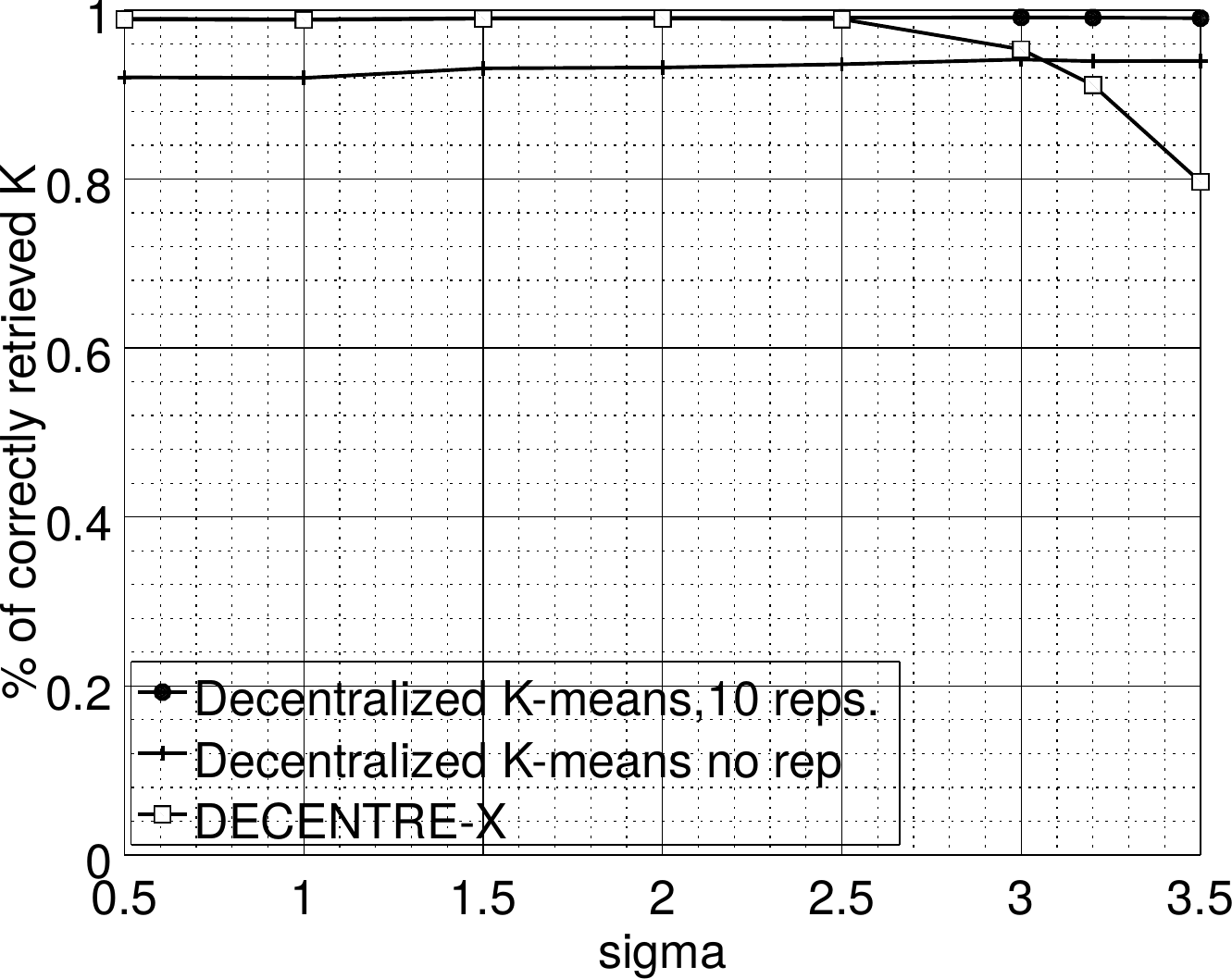}}
			\subfloat[~]{ \includegraphics[width=.5\linewidth]{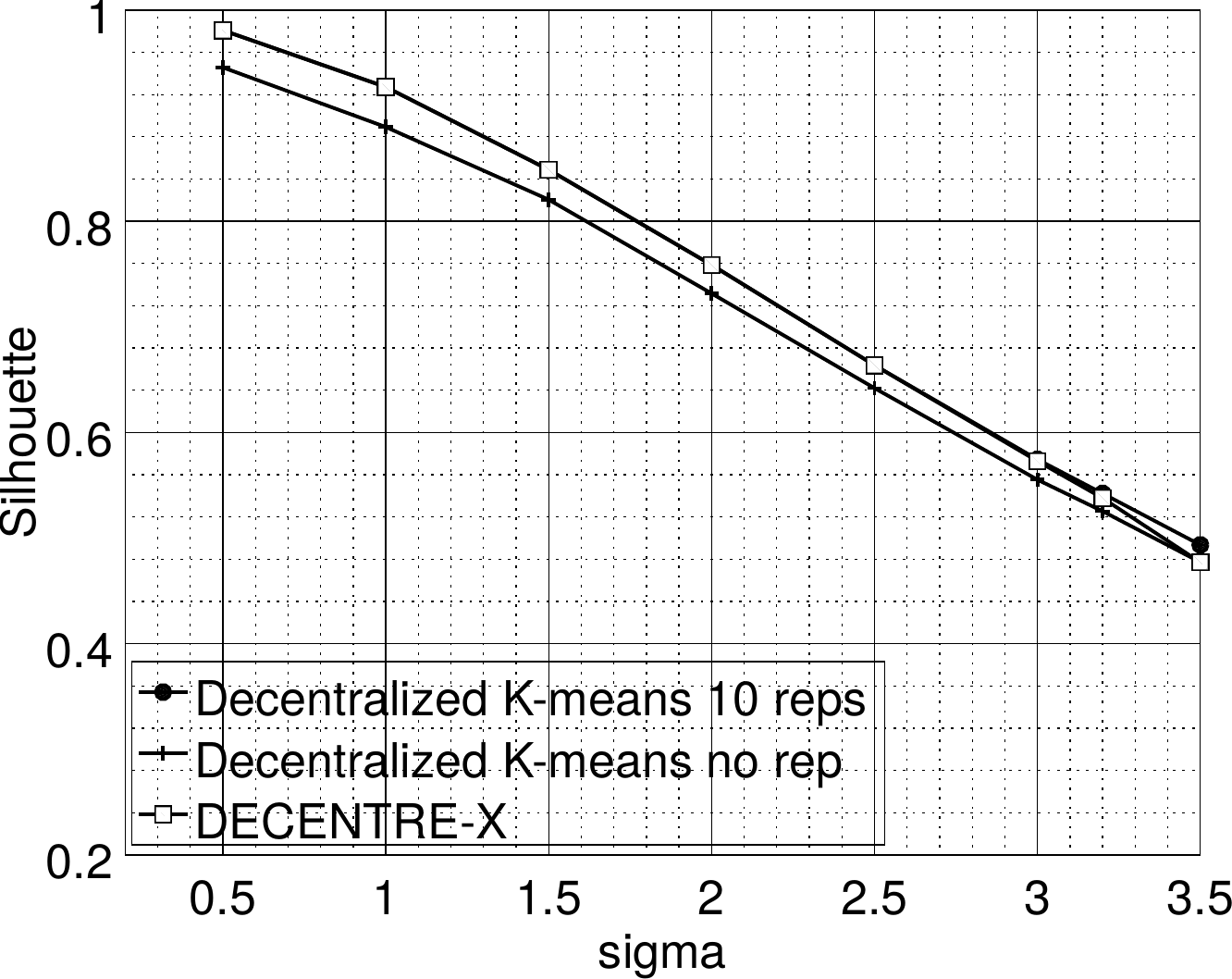}}
		\end{center}
		\caption{Performance with respect to $\sigma$ of \dalgo~compared to K-means with $R=10$ and $R=1$ for non-sparse centroids. Sigma is set to $2$. (a) Percentage of correctly retrieved  number of clusters (b) Silhouette}
		\label{fig:dX_nonsp_sigma}
	\end{figure*}
\end{onecol}

\begin{twocol}
\begin{figure}[t]
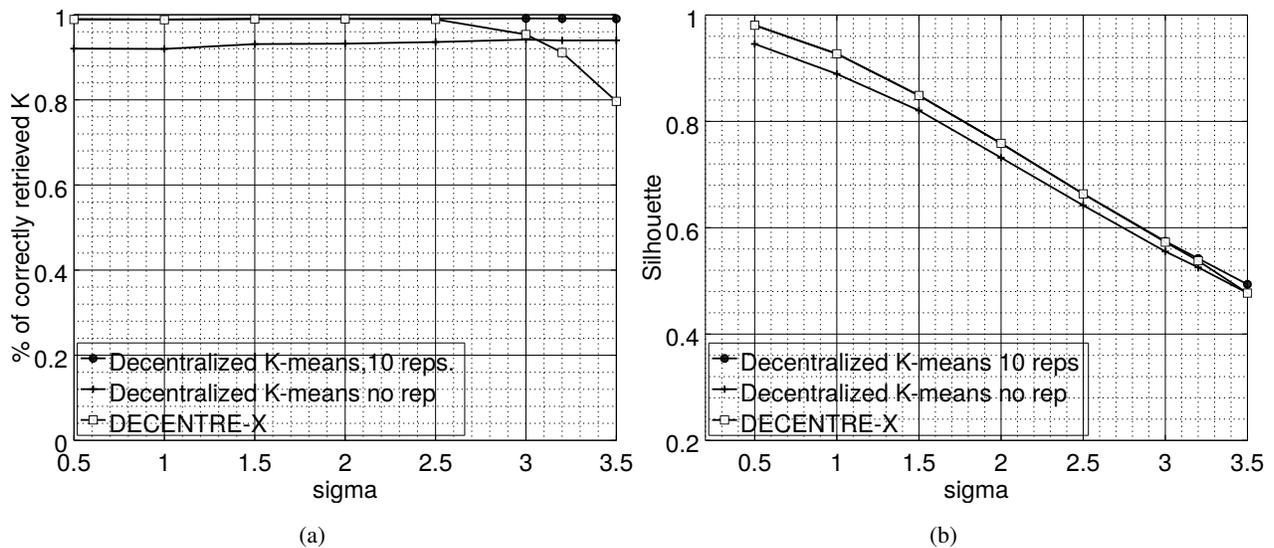

\begin{center}
  \subfloat[~]{ \includegraphics[width=.48\linewidth]{Figures/decentrex_nonsp_sigma_ke.pdf}}
  \subfloat[~]{ \includegraphics[width=.48\linewidth]{Figures/decentrex_nonsp_sigma_sil.pdf}}
\end{center}
\caption{Performance with respect to $\sigma$ of \dalgo~compared to K-means with $R=10$ and $R=1$ for non-sparse centroids. Sigma is set to $2$. (a) Percentage of correctly retrieved  number of clusters (b) Silhouette}
\label{fig:dX_nonsp_sigma}
\end{figure}
\end{twocol}

\vspace{-0.1cm}
\section{Conclusion \& perspectives}\label{sec:conclusion}
This paper has introduced \algo~and~\dalgo~for clustering compressed data over a network of sensors.
\algo~is a centralized algorithm that requires a fusion center, whereas \dalgo~is the fully decentralized version of \algo. 
These algorithms do not require prior knowledge of the number of clusters and do not suffer from initialization issues, which is highly beneficial in the decentralized setup of~\dalgo. 
The features satisfied by \algo~and \dalgo~follow from a novel theoretical framework that has introduced and established properties of a new type of cost function. Another originality of the approach is the use
of Wald's test p-value 
as the weight function involved in the cost function. 
Experimental results have shown that our algorithms are competitive in terms of clustering performance with respect to K-means with K unknown, while reducing the amount of data exchanged between sensors in the fully decentralized setup.

The new approach we have introduced for clustering relies on a statistical model of the measurements and can be adapted to other signal models. This may allow for addressing clustering problems that standard algorithms such as K-means can hardly handle. For instance, we could consider heterogeneous sensors that collect measurement vectors with different covariance matrices from one sensor to another, as in~\cite{SSP2018} for spectral clustering. 
We could also model the measurement vectors of a given cluster as random vectors with unknown distributions and known bounded variations corrupted by centered Gaussian noise.
For such models, extensions of the Wald tests are given in\cite{RDT, RDTlm}. 

\section{Acknowledgment}
The authors are very thankful to Oskar A. Rynkiewicz for the fruitful discussions on the algorithms presented in this paper.  

\section*{Appendices}

\appendices
\section{}
\label{App: asymptotic behaviors}
\begin{Lemma}
\label{Lemma: asymptotic behaviors}
\pastorv{Let $\matcov$ be any positive-definite matrix and set $\VZ{\xivec} \thicksim \Ncal(\xivec,\cCorMat)$ for any $\xivec \in \Rset^m$. With the same notation as in Section \ref{sec:model}, we have:} 
\medskip \\
\noindent
\pastorv{
(i) if $w: [0,\infty) \to [0,\infty)$ is bounded and such that
$\displaystyle \lim\limits_{t \rightarrow \infty} {w(t)} = 0$
then $\lim\limits_{ \| \xivec \| \rightarrow \infty} \! \! \weight ( \VZ{\xivec} ) = 0 \, \, \text{(a-s)}$
}
\medskip \\
\noindent
\pastorv{(ii) If $\lim\limits_{ \| \xivec \| \rightarrow \infty} \! \! \weight ( \VZ{\xivec} ) = 0 \, \, \text{(a-s)}$ then 
		\vspace{-0.1cm}
	$$\displaystyle \lim\limits_{ \| \xivec \| \rightarrow \infty} \Expect{w ( \nu_{\matcov}^2 \left ( \VZ{\xivec} \right ) ) \VZ{\xivec}} = 0$$
}
\end{Lemma}
\vspace{-0.25cm}
\begin{IEEEproof}
\medskip \\
\noindent
{\em Proof of statement (i):} \pastorv{Given $\xivec\in \Rset^\dimc$, $\VZctrd = \VZ{\xivec} - \xivec$ does not depend on $\xivec$ and $\VZctrd \thicksim \Ncal(0,\cCorMat)$. When $\| \xivec \|$ tends to $\infty$, $\nu_\matcov(\xivec)$ tends to $\infty$ as well, because all the norms are equivalent on $\Rset^\dimc$. Since $\nu_\matcov(\VZ{\xivec}) \geqslant \nu_\matcov(\xivec) - \nu_\matcov(\VZctrd)$ and $\weight(\VZ{\xivec}) = g(\nu_\matcov(\VZ{\xivec})$ with $g(t) = w(t^2)$ for any $t \in \Rset$, it follows that $\nu_\matcov(\VZ{\xivec})$ tends to $0$ when $\| \xivec \|$ tends to $\infty$.} 
\medskip \\
\noindent
{\em Proof of statement (ii):} \pastorv{With the same notation as above, 
	\begin{equation}
	\label{Eq: Basic equality for (iv)}
	\! \Expect{\weight ( \VZ{\xibm} ) \VZ{\xibm}} \! = \! \Expect{\weight ( \VZ{\xibm} )} \xivec \! + \! \Expect{\weight ( \VZ{\xibm} ) \VZctrd}.\! \! \! 
	\end{equation}
Let us consider the first term in the rhs of \eqref{Eq: Basic equality for (iv)}. Let $\MDec$ be the whitening matrix of $\cCorMat$. It follows from \eqref{Eq: Whitening matrix} that
$\VW{\xivec} = \MDec \VZ{\xivec}$ has distribution $\VW{\xivec} \thicksim \Ncal(\MDec \xivec, \Idm)$. Now, note that: 
\begin{equation}
\label{Eq: expect}
\Expect{ \weight \left (\VZ{\xivec} \right ) } \xivec = \MDec^{-1} \Expect{ \weight \left ( \MDec^{-1} \VW{\xivec} \right ) } \MDec \xivec.
\end{equation}
By setting $\zetavec = \MDec \xivec = ( \zeta_1, \zeta_2, \ldots, \zeta_m )^\transpose$, we have:
\begingroup\small
\begin{equation}
\label{Eq: inequalities for (iv)}
\Expect{ \weight \! \left ( \! \MDec^{-1} \VW{\xivec} \right ) } \! \zeta_i \! = \! \dfrac{1}{(2 \pi)^{\frac{m}{2}}} \! \displaystyle  \int \! \! \weight \! \left ( \MDec^{-1} \xvec \right ) \! \zeta_i e^{-\frac{\| \xvec - \zetavec \|^2}{2}} \! \d \xvec \!
\end{equation} \endgroup
for any $i = 1, 2, \ldots, \dimc$.
Given any fixed $\xvec \in \Rset^\dimc$, the inequality $\vert \zeta_i \vert \leqslant \| \xivec \|$ induces that:
$$
	- \| \xivec \| e^{-\frac{1}{2} \| \xvec - \zetavec \|^2} \leqslant \zeta_i e^{-\frac{1}{2} \| \xvec - \zetavec \|^2} \leqslant \| \xivec \| e^{-\frac{1}{2} \| \xvec - \zetavec \|^2}.
$$
whose left and right bounds tend to $0$ when $\| \xivec \|$ --- and thus $\| \zetavec \|$ thanks to the properties of $\MDec$ --- tends to $\infty$. By applying the Lebesgue dominated convergence theorem to \eqref{Eq: inequalities for (iv)} and since $i$ is arbitrary in $\llbracket 1,\dimc \rrbracket$, it follows from \eqref{Eq: expect} that:
\begin{equation}
\label{Eq: lim of expect1}
	\lim\limits_{\| \xivec \| \to \infty} \Expect{ \weight \left (\VZ{\xivec} \right ) } \xivec = 0.
\end{equation}
}
\indent
As far as the second term to the rhs of \eqref{Eq: Basic equality for (iv)} is concerned, we set \pastorv{$\VZctrd = (Z^*_1, \ldots, Z^*_\dimc)^\transpose$}. We have \pastorv{$\weight ( \VZ{\xivec} ) Z^*_i \leqslant \vert Z^*_i \vert$ for any $i \in \llbracket 1,\dimc \rrbracket$, where we assume, without loss of generality, that the bound on $w$ is $1$. Since $\Expect{\vert \, Z^*_i \, \vert} < \infty$ for each $i \in \llbracket 1,\dimc \rrbracket$ and $\displaystyle \lim\limits_{\| \xivec \| \rightarrow \infty} {\weight ( \VZ{\xivec} ) } = 0$ (a-s)}, we derive from the foregoing and the Lebesgue dominated convergence theorem that $\lim\limits_{\| \xivec \| \to \infty} \Expect{\weight ( \VZ{\xivec} ) Z^*_i} = 0$ for all $i \in \llbracket 1,\dimc \rrbracket$ and thus, that $\lim\limits_{\| \xibm \| \to \infty} \Expect{\weight ( \VZ{\xivec} ) \VZctrd } = 0$.
Thence the \pastorv{result} as a consequence of this equality, \eqref{Eq: Basic equality for (iv)} and \eqref{Eq: lim of expect1}. 
\end{IEEEproof}


\section{}\label{App:D}

\begin{Lemma}
\label{Lemma: last useful lemma}
\pastorv{Let $\Zbm$ be a Gaussian $\dimc$-dimensional real random vector with covariance matrix $\sigma^2 \Idm$ and $\sigma \neq 0$}. If $\fbm: \Rset^\dimc \to [0,\infty)$ is non-null, continuous and even in each coordinate of $\xbm = (x_1, \ldots, x_\dimc)^\transpose \in \Rset^\dimc$ so that $\fbm(x_1, \ldots,x_{i-1},x_i,x_{i+1}, \ldots, x_\dimc) = \fbm(x_1, \ldots,x_{i-1},-x_i,x_{i+1}, \ldots, x_\dimc),$  then $\Expect{\fbm ( \Zbm ) \Zbm} = 0$ if and only if $\Expect{\Zbm} = 0$.
\end{Lemma}

\begin{IEEEproof}
Suppose first that $\dimc = 1$. In this case, $\Zbm$ is a random variable $Z \thicksim \Ncal(\xi,\pastorv{\sigma^2})$ with $\xi = \Expect{Z}$ and $\fbm$ is \pastorv{even} nonnegative real function $f: \Rset \to [0,\infty)$. It follows that 
\begin{equation}
\label{Eq: EfZ)Z}
\Expect{f(Z)Z} = \frac{1}{\sqrt{2 \pi}} \int_{-\infty}^\infty f(z) z e^{-{(z-\xi)^2}/{2 \pastorv{\sigma^2}}} \dz
\end{equation} 
If $\xi = 0$, $\Expect{f(Z)Z} = 0$. If $\xi \neq 0$, split the integral in \eqref{Eq: EfZ)Z} in two, symmetrically with respect to the origin. After the change of variable $t = -z$ in the integral from $-\infty$ to $0$ resulting from the splitting, some routine algebra leads to $\Expect{f(Z)Z} = \frac{1}{\sqrt{2 \pi}} \int_{0}^\infty f(t) t e^{-{(t^2 + \xi^2)}/{2 \pastorv{\sigma^2}}} \left ( e^{\xi t / \pastorv{\sigma^2}} - e^{-\xi t / \pastorv{\sigma^2}} \right ) \dt$.  The integrand in this integral is continuous and has same sign as $\xi$, which implies that $\Expect{f(Z)Z} \neq 0$. Thence the result if $\dimc = 1$. 

In the $\dimc$-dimensional case, set $\Zbm = (Z_1, Z_2, \ldots, Z_\dimc)$ and denote the expectation $\Expect{\Zbm}$ of $\Zbm$ by $\xibm = (\xi_1, \xi_2, \ldots, \xi_\dimc)$. Let $\fbm: \Rset^\dimc \to [0,\infty)$ be a nonnull continuous function, even in each coordinate of $\xvec \in \Rset^\dimc$. Clearly, $\Expect{f(\Zbm) \Zbm} = 0$ if and only if $\Expect{f(\Zbm) Z_i} = 0$ for each coordinate $Z_i$, $i \in \{1, 2, \ldots, \dimc \}$. For the first coordinate $Z_1$ of $\Zbm$, it follows from Fubini's theorem that:
\[
\vspace{-0.1cm}
\Expect{\fbm ( \Zbm ) Z_1} =\displaystyle \int f(z_1) \, z_1 \, e^{-(z_1 - \xi_1)^2/2\pastorv{\sigma^2}} \dz_1
\]
with $f: \Rset \to [0,\infty)$ defined for any real number $z_1$ by:
\begingroup \small
$$\vspace{-0.1cm}
f(z_1) \! = \! \dfrac{1}{(2 \pi)^{\dimc/2}} \! \displaystyle \int \! \fbm(z_1, \ldots, z_\dimc) e^{-\sum_{k=2}^\dimc (z_k - \xi_k)^2/2 \pastorv{\sigma^2}} \dz_2  \ldots  \dz_\dimc.$$ \endgroup
This function is continuous, non-negative and even. The conclusion then follows from the one-dimensional case.
\end{IEEEproof}

\section{p-value of Wald's test for Gaussian mean testing}
\label{Sec: RDT pvalue}
Fort any $x \in [0,\infty)$, we hereafter set $\MyMarcum(x) = \Marcum(0,x)$. Given $\level \in (0,1)$, the value $\mu(\level)$ defined by \eqref{Eq: Wald threshold} is the unique real value such that $\MyMarcum(\mu(\level)) = \level$. 
\begin{Lemma}
\label{Lemma: tsub decreases}
Given $\tol \in [ \, 0 \, , \, \infty \, )$, the map $\level \in ( \, 0 \, , \, 1 \, ] \mapsto \mu(\level) \in [ \, 0 \, , \infty \, )$ is strictly decreasing.
\end{Lemma}
\begin{proof}
Let $\radius \in [0,\infty)$ and consider two elements $\level$ and $\level'$ of $(0,1]$. We have $\MyMarcum(\mu(\level)) = \level$ and $\MyMarcum(\mu(\level')) = \level'$. If $\level < \level'$, $\MyMarcum(\mu(\level)) < \MyMarcum(\mu(\level'))$, which implies that $\mu(\level) > \mu(\level')$ since $\MyMarcum$ is strictly decreasing \cite{Sun2010}.
\end{proof}

Suppose that $\Ybm \thicksim \Ncal(\xivec,\matcov)$ with $\xivec \in \Rset^\dimc$ and $\matcov$ is an $\dimc \times \dimc$ positive definite covariance matrix. The critical region of the test $\Topt$ defined by \eqref{Eq:Thresholding test from above}
is:
$$
\Scal_\level = \mybig \{ \ybm \in \Rset^\dimc: \Topt(\ybm) = 1 \mybig \} = \mybig \{ \ybm \in \Rset^\dimc: \thenorm( \ybm ) > \mu(\level) \mybig \}
$$
According to Lemma \ref{Lemma: tsub decreases}, for two levels $0 < \level < \level' < 1$, we have $\Scal_{\level} \subset \Scal_{\level'}$. Given $\ybm \in \Rset^\dim$, we can thus define the p-value of $\Topt$ at $\ybm$ as \cite[p. 63, Sec. 3.3,]{Lehmann2005} $\pval{\ybm} = \inf \mybig \{ \level \in (0,1): \ybm \in \Scal_\level \mybig \}$.
If $\level_0 = \MyMarcum(\thenorm(\ybm))$, we have $\level_0 \in (0,1)$ and $\MyMarcum(\mu(\level_0)) = \level_0$ by definition of $\mu(\level_0)$. It then follows from the bijectivity of $\MyMarcum$ that $\mu(\level_0) = \thenorm( \ybm )$. According to Lemma \ref{Lemma: tsub decreases} again, we obtain $\mybig \{ \level \in (0,1): \mu(\level) < \thenorm( \ybm ) \mybig \} = (\level_0,1) $.
Therefore $\pval{\ybm} = \MyMarcum(\thenorm( \ybm ))$. 


\section{}
\label{App: Correlation matrix}
\begin{Lemma}
\label{Lemma: Covmat}
$\Expect{w^2 \left ( \| \Noiseb \|^2 \right ) \Noiseb \Noiseb^\transpose} = \Expect{w^2 \left ( \| \Noiseb \|^2 \right )\noiseb^2_1} \Idm$  for any $\Noiseb = (\noiseb_1, \ldots, \noiseb_\dimc)^\transp\thicksim \Ncal(0, \coeff\Idm)$, 
\end{Lemma}
\begin{proof}
The term located at the $i$th line and $j$th colum of the matrix $\Expect{w^2 \left ( \| \Noiseb \|^2 \right ) \Noiseb \Noiseb^\transpose}$ with $i \ne j$ is:
\begin{onecol}
$$
c_{i,j} 
= \Expect{w^2 \left ( \| \Noiseb \|^2 \right )\noiseb_i \noiseb_j} 
= \dfrac{1}{(2\pi \coeff)^{\dimc/2}} \displaystyle \int_{\Rset^\dimc} w^2  \Big ( \sum_{k=1}^m \xi_k^2 \Big ) \xi_i \xi_j e^{-\sum_{k=1}^m \xi_k^2/2\coeff} \d \xi_1 \d \xi_2 \ldots \d \xi_\dimc.
$$
\end{onecol}
\begin{twocol}
\vspace{-0.1cm}
	$$
	c_{i,j} 
\!	= \! \dfrac{1}{(2\pi \coeff)^{\frac{\dimc}{2}}} \displaystyle \! \int \! \! w^2  \Big (  \sum_{k=1}^m \xi_k^2 \Big ) \xi_i \xi_j e^{-\sum_{k=1}^m \xi_k^2/2\coeff} \d \xi_1 \ldots \d \xi_\dimc.
	$$
\vspace{-0.1cm}
\end{twocol}
By independence of the components of $\Noiseb$ and Fubini's theorem, we have:
\begin{onecol}
$$
c_{i,j} 
= \dfrac{1}{(2\pi \coeff)^{\dimc/2}} \displaystyle \int \prod_{k=1, k\ne i,j}^\dimc e^{-\xi_k^2/2\coeff} \d \xi_k \displaystyle \int \left ( \, \displaystyle \int w^2 \Big ( \sum_{k=1}^m \xi_k^2 \Big ) \xi_i e^{-\xi_i^2/2\coeff} \d \xi_i \, \right ) \xi_j e^{-\xi_j^2/2\coeff} \d \xi_j
$$
\end{onecol}
\begin{twocol}
\vspace{-0.15cm}
	\begin{align}  \nonumber
	& c_{i,j} =  \vspace{-0.1cm} \\ \nonumber
	&  \! \displaystyle \int \! \! \! \underset{k\ne i,j}{\prod_{k=1,}^\dimc} p(\xi_k) \d \xi_k \displaystyle \int \left ( \, \displaystyle \int w^2 \Big ( \sum_{k=1}^m \xi_k^2 \Big ) \xi_i p(\xi_i) \d \xi_i \, \right ) \xi_j p(\xi_j) \d \xi_j
	\vspace{-0.1cm}
	\end{align}
	with $p(x) = ({1}{\sqrt{2\pi \coeff}})e^{-x^2/\coeff}$ for $x \in \Rset$. 
\end{twocol}
The integrand is odd in $\int \! w^2 ( \sum_{k=1}^m \xi_k^2 ) \xi_i p(\xi_i) \d \xi_i \! = \! 0$ and thence, $c_{i,j} = 0$.
\end{proof}

\bibliographystyle{ieeetr}
\bibliography{IEEEabrv,Refs}

\end{document}